\documentclass[11pt,a4paper,reqno]{amsart}
\usepackage{graphicx}
\usepackage{epsfig}
\usepackage{cite}
 \newtheorem{thm}{Theorem}[section]
 \newtheorem{coro}[thm]{Corollary}
 \newtheorem{lemma}[thm]{Lemma}
 
  \newtheorem{conj}[thm]{Conjecture} 
 \theoremstyle{definition}
 \newtheorem{defn}[thm]{Definition}
 \theoremstyle{remark}
 
 \usepackage{amsbsy,amsmath,amsthm,hyperref,xcolor,algorithm,algpseudocode,amssymb}
 \algrenewcommand\algorithmicrequire{\textbf{Input:}}
\algrenewcommand\algorithmicensure{\textbf{Output:}}

\newcommand{\bsy}{\boldsymbol}
\newcommand{\ds}{\displaystyle}
\newcommand{\cal}{\mathcal}

 \DeclareMathOperator*{\argmax}{arg\,max}

\begin{document}
\setcounter{page}{1}
\begin{flushleft}
{\scriptsize Appl. Comput. Math., V.xx, N.xx, 20xx, pp.xx-xx}
\end{flushleft}
\bigskip
\bigskip
\title[On the Stability of a non-hyperbolic nonlinear map]{On the Stability of a non-hyperbolic nonlinear map \\ with non-bounded set of non-isolated fixed points with\\ Applications to Machine Learning}

\author[Appl. Comput. Math., V.xx, N.xx,  20xx]{R. Hansen$^1$, M. Vera$^{1,2}$, L. Estienne$^{1,3}$, L. Ferrer$^3$ and P. Piantanida$^4$}
\thanks{$^1$Universidad de Buenos Aires, Facultad de Ingeniería, Buenos Aires, Argentina,
\\ 
\indent $^2$CONICET, Centro de Simulación Computacional para Aplicaciones Tecnológicas (CSC), Buenos Aires, Argentina,\\
\indent $^3$CONICET, Instituto de investigación en Ciencias de la Computación (ICC), Buenos Aires, Argentina,\\
\indent $^4$International Laboratory on Learning Systems (ILLS) and Quebec AI Institute (MILA), CNRS and CentraleSupélec - Université Paris-Saclay, Quebec, Canada,\\
\indent\,\,\,e-mail: rhansen@fi.uba.ar, mvera@fi.uba.ar, lestienne@fi.uba.ar, lferrer@dc.uba.ar, pablo.piantanida@cnrs.fr
\\ \indent
  \em \,\,\,Manuscript received xx}

\begin{abstract}
This paper deals with the convergence analysis of the SUCPA (Semi Unsupervised Calibration through Prior Adaptation) algorithm, defined from a first-order non-linear difference equations, first developed to correct the scores output by a supervised machine learning classifier. The convergence analysis is addressed as a dynamical system problem, by studying the local and global stability of the nonlinear map derived from the algorithm. This map, which is defined by a composition of exponential and rational functions, turns out to be non-hyperbolic with a non-bounded set of non-isolated fixed points. Hence, a non-standard method for solving the convergence analysis is used consisting of an ad-hoc geometrical approach.  For a binary classification problem (two-dimensional map), we rigorously prove that the map is globally asymptotically stable. Numerical experiments on real-world application are performed to support the theoretical results by means of two different classification problems: Sentiment Polarity performed with a Large Language Model and Cat-Dog Image classification. For a greater number of classes, the numerical evidence shows the same behavior of the algorithm, and this is illustrated with a Natural Language Inference example. The experiment codes are publicly accessible online at the following repository: \texttt{https://github.com/LautaroEst/sucpa-convergence}.




\end{abstract}
\maketitle

\section{Introduction}\label{Intro}

Difference equations have gained a lot of attention in the last decades due to their applicability in many branches of scientific knowledge. These equations model discrete natural and social phenomena and perform a fundamental role in many applications like population dynamics, control engineering, genetics, signal processing, health sciences and ecology \cite{Khan}. In most cases, difference equations arise naturally as discretization of differential equations \cite{Elaydi1} or are defined by an iterative algorithm that describe a dynamic system. In the later case, the algorithm can be fully described by a map (i.e., a function) and the nature of this map defines the convergence properties of the algorithm. As consequence, studying the map and the difference equation associated to a dynamic system is usually considered a very important topic because they define the behaviour of the system's dynamic.  

\subsection{Non-hyperbolic maps and fixed points}

Depending on the discipline in which they are applied, different methods were developed throughout the literature to demonstrate the convergence of iterative algorithms (e.g. \cite{AAH,LiCh,Tar}). In most cases, the map associated with the algorithm is hyperbolic and the solutions of its associated difference equation (or equivalently, the points to which the algorithm converges) form a finite set of points --and therefore isolated. This points are usually referred to as the \emph{fixed points} of the system. Roughly speaking, a fixed point is called \emph{isolated} if there exists a neighborhood of it that does not contain any other fixed points of the system, otherwise it is called non-isolated. Also, a fixed point of a map is called \emph{hyperbolic}, if the Jacobian matrix of the map at this point has no eigenvalues of modulus one, otherwise, it is called non-hyperbolic \cite{Devaney,Elaydi2}. 

As it is known, the eigenvectors provide the contracting/expanding directions of the local dynamics according to whether their corresponding eigenvalues have modulus greater/lower than one. The behaviour of dynamical systems around a non-hyperbolic fixed point is much more subtle than for a hyperbolic one, because one cannot get a definite conclusion of the local stability from just the linearization of the map when it has eigenvalues in the unit circle. These are called \emph{resonant} cases and the scenarios are very different depending on whether the eigenvalues are 1, -1 or a pair of complex conjugates, and on the amount of each of them \cite{Kuznetsov}. Neither is there a general classification of them, so a case-by-case analysis is required \cite{Bauerschmidt}. The local stability near a non-hyperbolic fixed point is addressed by the theory of Central Manifold,  which is an invariant set in a low dimensional space where the local dynamics can be reduced. This theory is relevant and mostly applied in the case of a system's parameter bifurcation. It consists of transforming the map into its normal form, first by setting the fixed point to the origin through a suitable change of coordinates, and then splitting up the algebraic expression of the new map into a linear and a nonlinear terms (the interested readers on this theory can see \cite{Carr, Elaydi2, Kuznetsov,Wiggins}). Thus, this formulation is used mostly in theoretical problems, or successfully applied when the map under study posses a relative simple algebraic expression (polynomial, exponential, etc.) as in \cite{Psarros}; otherwise it is very cumbersome to be utilized.

In this work, we study the convergence of the recently proposed SUCPA algorithm \cite{sucpa_arxiv} by means of the local and global stability analysis of the map derived from it. This $K$-dimensional nonlinear map ($K\!\geq\!2$) is \emph{non-hyperbolic} and its fixed points form an \emph{unbounded} set of \emph{non-isolated} points (indeed, they are a straight line). We prove the convergence for the two-classes case, and show different properties and conjectures in the general case.  

The map associated with the SUCPA algorithm consists of equations that involve a composition of  exponential and rational functions, and the non-hyperbolic fixed points have only one eigenvalue equal to one and the rest of them have modulus less than one. The particular study of --high order-- rational difference equations like the ones present in this map has gained a lot of attention in the last time, because there is no an effective general method to deal with, so they  currently are a class of challenging problems, and only specific type of them are treated and reported in the literature \cite{Elsayed,Stevic,Ibrahim}. In addition, despite that the study of non-hyperbolic maps is addressed in several works, they do not usually bring together all the features presented here. In most cases they are about isolated fixed points \cite{Schaumann}, one-dimensional maps \cite{Kapcak,Dannan,Urbanski}, flows \cite{Zolfaghari}, area-preserving maps \cite{Liu-Chen} or a two dimensional map with double eigenvalue equal one \cite{Jamieson}, to name a few. 

Another problem addressed in this work is the global stability of the map, which would account for the convergence of the SUCPA algorithm for any input. The standard tools for this purpose --as obtaining a suitable Lyapunov function-- cannot be applied here either for the same reason that the fixed points are not isolated. This fact, together with the no so simple algebraic equations of the system, force us to prove the local and global stability by non-standard ways, defining auxiliary functions and solving it by an ad-hoc geometrical approach.

\subsection{SUCPA as a calibration algorithm}

The SUCPA algorithm was first developed in the context of supervised machine learning classifiers. A supervised machine learning classifier is a function $\hat{h}(\cdot)$ that usually maps an input $\mathbf{x}$ composed by real-world data (e.g., images, text or audio) to an output $y$ that can be of the same type as the input or could be numerical. An example of these functions is an image classifier, which takes in an image in the form of a 3D tensor $\mathbf{x}\! \in\! \mathbb{R}^{C\times N \times M}$ and outputs the category $y$ to which that image belongs. (For instance, $y\! = \!\{y_1,y_2\}\! = \!\{\mathrm{cat},\mathrm{dog}\}$ in the case of binary classification of images of animals). In most cases, classifiers map the input to a category contained in an unordered finite set of discrete categories.

The goal of supervised learning is not to manually design the function $\hat{h}(\cdot)$, but to obtained it by exposing the system to a set $\mathcal{D}\!=\!\{(\mathbf{x}^{(1)},y^{(1)}),\hdots,(\mathbf{x}^{(N)},{y}^{(N)})\}$ of multiple input-output examples (called, the \emph{training set}). More formally, the function $\hat{h}(\cdot)$ is obtained by minimizing a loss $\mathcal{L}_{\mathcal{D}}(h)$ function over the entire set of possible functions considered in the problem. Then, this function is used to predict an output $\hat{y}\!=\!\hat{h}(\mathbf{x})$ from a new input $\mathbf{x}$. Particularly in the case of probabilistic machine learning classifiers, the training set is used to obtain a so called \emph{posterior probability distribution} $P(y|\mathbf{x})$ over the set $\mathcal{Y}=\{y_1,\ldots,y_K\}$ of possible categories, for each possible input $\mathbf{x}$. It is a well known result that once this distribution is known, it is trivial to map the input $\mathbf{x}$ to the label $\hat{y}$ by making Bayes decisions, which are those that minimize the Bayes risk \cite{duda}, a standard criterion for measuring performance of a classifier. A well known instance of this risk (and the one used in this work) is the one obtained by assigning equal cost to each decision. In this case, it can be shown that the risk is minimized by choosing the class $\hat{y}$ that maximizes the posterior $P(y|\mathbf{x})$:
\begin{equation}
    \hat{y} = \hat{h}(\mathbf{x}) = \argmax_{y \in \mathcal{Y}} P(y|\mathbf{x})
\end{equation}

Concerning the quality of the estimation of the posterior $P(y|\mathbf{x})$, it has been studied in the last years that the posterior distribution produced by most modern classifiers is miscalibrated \cite{guo17}, which means that the Bayes decisions for that posteriors are not optimal for all cost functions \cite{survey_calibration}\footnote{A more formal definition of calibration is that a system is calibrated if the estimated probability of the predicted label matches with the true probability of that category. For more details, consult \cite{survey_calibration}.}. This may happen for a few different reasons. A model that overfitted the training data tends to overestimate its certainty about the class, resulting in posteriors that take sub-optimally extreme values \cite{guo17}. Further, when the data distribution of the training and the test data are different, the model may also be miscalibrated. In particular, when the training and test class probabilities (so-called \emph{priors probability distributions}) are different, the posteriors will be sub-optimal since the priors are implicitly encoded in the posteriors \cite{godau}.

Many approaches have been proposed in the literature to calibrate machine learning systems that are miscalibrated. One simple calibration method is based on logistic regression \cite{Brmmer2010MeasuringRA}, taking as input features the log posteriors produced by the model (i.e. the logarithm of the estimated probability $P(k|\mathbf{x})$ of a class  $k$, $k\!=\!1,\ldots,K$, given an observed input $\mathbf{x}$) and applying an affine transformation of the following form: 
\begin{equation}\label{eq:affine_log}
\log\tilde P(k|\mathbf{x})=\alpha_k\log P(k|\mathbf{x})+\beta_k+\gamma(\boldsymbol{\alpha},\boldsymbol{\beta}),
\end{equation}
where $\tilde P(k|\mathbf{x})$ is the calibrated posterior probability, $\bsy{\alpha}\!=\!\big[\alpha_1,\hdots,\alpha_K\big]$, $\bsy{\beta}\!=\!\big[\beta_1,\hdots,\beta_K\big]$ and $\gamma(\boldsymbol{\alpha},\boldsymbol{\beta})$ is determined so that $\sum_{k=1}^K\tilde P(k|\mathbf{x})=1$. Here, $\boldsymbol{\alpha}$ and $\boldsymbol{\beta}$ are parameters of the affine transformation, and they are trained to minimize a Proper Scoring Rule (usually the Negative Log-Likelihood or \emph{cross-entropy}) \cite{survey_calibration}. 

Different assumptions in \eqref{eq:affine_log} are usually made according to the specific problem that needs to be solved. For instance, temperature scaling \cite{guo17}, one of the most widely used calibration methods, corresponds to taking $\beta_k\! =\! 0$ for all $k$ and $\alpha_k\!=\!\frac{1}{T}$ a single scalar. In recent work, we proposed the SUCPA algorithm to address the scenario of mismatched priors for text classification tasks solved with large language models~\cite{sucpa_arxiv}. This algorithm is derived for the case in which $\alpha_k\!=\!1$ for all $k \in \{ 1,\ldots,K\}$, which is equivalent to assume that the source of miscalibration is only due to a mismatch in class priors. 
As logistic regression, SUCPA is obtained by minimizing the cross-entropy loss on the training set $\mathcal{D}\!=\!\{(\mathbf{x}^{(1)},y^{(1)}),\hdots,(\mathbf{x}^{(N)},{y}^{(N)})\}$:
\begin{align}
\mathcal{L}(\boldsymbol{\beta})&=\frac{1}{N}\sum_{i=1}^N-\log\tilde P(y^{(i)}|\mathbf{x}^{(i)}) =-\frac{1}{N}\sum_{i=1}^N\left[\log P(y^{(i)}|\mathbf{x}^{(i)})+\beta_{y^{(i)}}+\gamma_i(\mathbf{1},\boldsymbol{\beta})\right]
\end{align}
If we set the derivative of the cross-entropy to zero, we can derive the following expression for $\beta_k$:
\begin{equation}\label{eq:sucpa}
\beta_k=\log\left(\frac{N_k}{N}\right)-\left(\frac{1}{N}\sum_{i=1}^N\frac{P(y_k|\mathbf{x}^{(i)})}{\sum_{j=1}^KP(y_j|\mathbf{x}^{(i)})e^{\beta_j}}\right)
\end{equation}
where $N_k$ is the number of samples with $y^{(i)}=k$. Mathematical details can be read in \cite{sucpa_arxiv}. If the proportion $\frac{N_k}{N}$ is known as external knowledge from the nature of the task, then \eqref{eq:sucpa} can be used to iteratively estimate the value of $\beta_k$ for every $k\!\in\!\{1,\ldots,K\}$ using just the samples $\{\mathbf{x}^{(1)},\hdots,\mathbf{x}^{(N)}\}$ (i.e., with no labels). It is important to highlight the need to have this prior information, which limits the spectrum of possible applications to tasks where this knowledge is precise.

\smallskip

The rest of the paper is organized as follows. In Section \ref{sec:problem_description} we present the $K$-dimensional nonlinear map derived from the algorithm, and pose the main conjecture of the work. Section \ref{sec:properties} summarizes some essential properties of the map. In Section \ref{sec:two-classes} we prove the convergence for two classes ($K\!=\!2$). Numerical examples are shown in Section \ref{sec:numerical_examples}. Finally, in Section \ref{sec:conclutions}, some concluding remarks are discussed.

\smallskip


\bigskip
\section{The SUCPA-map}\label{sec:problem_description}

In this section, we formally introduce the $K$-dimensional map, $\mathbf{f}$, derived from the SUCPA algorithm presented in \cite{sucpa_arxiv}, henceforth the {\it SUCPA-map}. We also recall some useful definitions and formulate the main conjecture of the work in terms of this map. The discrete time is denoted by $t\!\in\!\mathbb{N}$.

\begin{defn}\label{def:mapa}
Let $\bsy{\beta}^{[t]}\!=\!\big[\beta_1^{[t]},\hdots,\beta_K^{[t]}\big]\!\in\!\mathbb{R}^{K}$ and $\bsy{\beta}^{[0]}\!\in\!\mathbb{R}^K$ be an initial condition. Let also, $[N_1,\hdots,N_K]\!\in\!\mathbb{N}^K$, with $\sum_{k=1}^KN_k\!=\!N$ and $\mathbf{P}\!\in\!\mathbb{R}^{N\times K}$ a matrix with coefficients $P_{i,k}\!>\!0$ and $\sum_{k=1}^KP_{i,k}=\!1$. We define the {\it SUCPA algorithm} as:
\begin{equation}\label{eq:algoritmo}
\beta_k^{[t+1]}=-\log\left(\frac{1}{N_k}\sum_{i=1}^N\frac{P_{i,k}}{\sum_{j=1}^KP_{i,j}e^{\beta_j^{[t]}}}\right),\quad 1 \leq k \leq K.
\end{equation}
\end{defn}

The {\it SUCPA-map}, $\mathbf{f}\!=\![f_1,\hdots,f_K]\!:\!\mathbb{R}^K\!\to\!\mathbb{R}^K$, is the map defined by \eqref{eq:algoritmo}, as $\bsy{\beta}^{[t+1]}\!=\!\mathbf{f}\big(\bsy{\beta}^{[t]}\big)$. Then:
\begin{equation}\label{eq:transition}
f_k(\bsy{\beta})=-\log\left(\frac{1}{N_k}\sum_{i=1}^N\frac{P_{i,k}}{\sum_{j=1}^KP_{i,j}e^{\beta_j}}\right),\quad 1 \leq k \leq K.
\end{equation}
As usual, $\mathbf{f}^{(t+1)}(\boldsymbol{\beta})\!=\!\mathbf{f}\big(\mathbf{f}^{(t)}(\boldsymbol{\beta})\big)$, $\forall\, t\!\in\!\mathbb{N}$, and $\mathbf{f}^{(0)}(\boldsymbol{\beta})\!=\!\boldsymbol{\beta}$. Note that Eq. (\ref{eq:transition}) is a system of $K$ first-order difference equations involving a composition of exponential and rational functions.
 
\begin{defn}\label{def:punto-fijo-w-limit}
Let $\mathbf{f}$ be a map, then \cite{Devaney,Wiggins}:
\begin{enumerate}
    \item $\bsy{\beta^\ast}\!\in\!\mathbb{R}^K$ is a {\it fixed point of $\mathbf{f}$} if it verifies $\mathbf{f}(\bsy{\beta^\ast})\!=\!\bsy{\beta^\ast}$.
    \item The {\it $\omega$-limit set} of $\bsy{\beta}$ is the set:
      \begin{equation} 
      \omega(\bsy{\beta})=\bigcap_{n=0}^\infty\overline{\{\mathbf{f}^{(t)}(\bsy{\beta}), t\!\geq\!n\big\}}
  \end{equation}  
{where the over-bar accounts for the closure of a set.}
    \end{enumerate}
 \end{defn}
 
The fixed points are those that kept invariant when undergoing iterations of the algorithm. The $\omega$-limit is the set toward the algorithm, started in $\bsy{\beta}$, converges with the successive iterations. Note that
$\omega(\bsy{\beta})$ is a limit set because it is the intersection of non-increasing sets. It is also a closed set (by definition), and it can be of different type --a single point, a periodic orbit, an infinite set, an unbounded set, an empty set, and it may have also complicated structures-- depending on the $\mathbf{f}$ structure \cite{Wiggins}.
\smallskip

As said in the introduction, this work deals with the convergence of the SUCPA algorithm for any initial condition (i.c.), $\bsy{\beta}^{[0]}\!\in\!\mathbb{R}^K$, which leads us to study the stability properties of the related SUCPA-map.

\begin{conj}\label{conj:convergencia}
For each i.c. $\boldsymbol{\beta}^{[0]}\!\in\!\mathbb{R}^K$, there exist a unique fixed point of the SUCPA-map, $\mathbf{f}$, $\boldsymbol{\beta^\ast}\!\!\in\!\mathbb{R}^K$ ($\boldsymbol{\beta^\ast}(\boldsymbol{\beta}^{[0]})$), such that $\omega(\boldsymbol{\beta}^{[0]})\!=\!\{\boldsymbol{\beta^\ast}\}$ (it is a single point set).   
\end{conj}
This result is completely demonstrated in the case of two dimensions, i.e. $K\!=\!2$, corresponding to a binary classification problem for the original algorithm. The case $K\!>\!2$ of this conjecture is not yet proven, however many interesting properties can be shown, and this will be done in the next section.


\section{General properties of the SUCPA-map}\label{sec:properties}

Within the literature of discrete dynamical systems, most problems deal with the stability of isolated fixed points, for which the standard tools to study both, local and global stability focus on the linearization of the system at those points, and/or obtaining the so-called Lyapunov functions \cite{Devaney,Wiggins}. Furthermore, the number of fixed points is usually finite --therefore they form a bounded set. The SUCPA-map defined in \eqref{eq:transition} actually presents a fairly simple behavior, in terms of its dynamics --for example, there is no presence of chaos despite being non-linear. However, as will be demonstrated, the fixed points of the system form a non-bounded set of non-isolated points. Indeed, the set of fixed points is a straight line. 
These two features together with the algebraic type of equations, make it necessary to study the stability of the system by alternative ways. We did this, for the case $K\!=\!2$, by defining properly auxiliary functions that allowed us to argue geometrically.
\subsection{Fixed points analysis}

In terms of the SUCPA-map $\mathbf{f}$ of Eq. \eqref{eq:transition}, and according to Def. \ref{def:punto-fijo-w-limit}(1), $\boldsymbol{\beta}^\ast\!=\![\beta_1^\ast,\hdots,\beta_K^\ast]$ is a fixed point of $\mathbf{f}$, if and only if, $\forall\,k\!=\!1,\cdots,K$:
\begin{equation}\label{eq:fixedpoints}
e^{-\beta_k^{\ast}}=\frac{1}{N_k}\sum_{i=1}^N\frac{P_{i,k}}{\sum_{j=1}^KP_{i,j}e^{\beta_j^{\ast}}}
\end{equation}
Equations \eqref{eq:fixedpoints} cannot be solved explicitly, so we must approach this point from another side.
\smallskip

The SUCPA-map has a particular behavior against constant vector sums. It is easy to see the following result.
\begin{lemma}\label{lem:singletransition}
Let $\bsy{\lambda}\!=\![\lambda,\hdots,\lambda]\!\in\!\mathbb{R}^K$ a constant vector (all entries the same), then for all $\bsy{\beta}\!\in\!\mathbb{R}^K$, the SUCPA-map satisfies:
\begin{equation}\label{eq:mapeo-recta}
\mathbf{f}(\bsy{\beta}\!+\!\bsy{\lambda})=\mathbf{f}(\bsy{\beta})+\bsy{\lambda}
\end{equation}
\end{lemma}
\begin{proof} It is straightforward evaluating Eq. \eqref{eq:transition} in $\bsy{\beta}\!+\!\bsy{\lambda}$.    
\end{proof}
Naturally, this property remains valid for successive iterations of $\mathbf{f}$ and it is the reason for the geometry structure of the fixed points set --previously subjected to the proof of its existence.

\begin{coro} \label{recta-S}
Let $\bsy{\lambda}\!=\![\lambda,\hdots,\lambda]\!\in\!\mathbb{R}^K$. If $\bsy{\beta^\ast}$ is a fixed point of $\mathbf{f}$, then the points in the straight line, ${\cal S}(\bsy{\beta^\ast})$, with $\bsy{\lambda}$ direction and through $\bsy{\beta^\ast}$:
\begin{equation}\label{eq:straight-line}
    {\cal S}(\bsy{\beta^\ast})=\big\{\bsy{\beta}=\bsy{\lambda}\!+\!\bsy{\beta^\ast}\big\}
\end{equation}
are also fixed points of $\mathbf{f}$.
\end{coro}

\begin{proof} It is straightforward by induction in $t\!\in\!\mathbb{N}$. Lemma \ref{lem:singletransition} is the case for $t\!=\!1$. Assuming $\mathbf{f}^{(t)}(\boldsymbol{\beta}\!+\!\bsy{\lambda})=\mathbf{f}^{(t)}(\boldsymbol{\beta})+\bsy{\lambda}$ is valid for $t\!>\!1$, then for $t\!+\!1$, Lemma \ref{lem:singletransition} implies:
\begin{equation}
\mathbf{f}^{(t+1)}\big(\bsy{\beta}\!+\!\bsy{\lambda}\big)=\mathbf{f}\,\big(\mathbf{f}^{(t)}(\bsy{\beta}\!+\!\bsy{\lambda})\big)=\mathbf{f}\big(\mathbf{f}^{(t)}(\bsy{\beta})\!+\!\bsy{\lambda}\big)=\mathbf{f}^{(t+1)}(\bsy{\beta})+\bsy{\lambda}
\end{equation}
\end{proof}
This result shows that, if for some i.c. $\bsy{\beta}^{[0]}$ the algorithm converges to $\bsy{\beta^\ast}$, then for a shifted i.c. $\bsy{\beta}^{[0]}\!+\!\bsy{\lambda}$,  it converges to $\bsy{\beta^\ast}\!+\!\bsy{\lambda}$ (an example of this behaviour can be seen in Fig. \ref{fig:ej-1-2D}). Therefore, if  $\bsy{\beta^\ast}$ is a fixed point of $\mathbf{f}$, the points of the straight line with $\mathbf{1}\!=\![1,\hdots,1]$ direction and through $\bsy{\beta^\ast}$ are also fixed points of $\mathbf{f}$. 
In fact, there exists abundant numerical evidence to formulate the following conjecture, although only the case $K\!=\!2$ is formally proven in Sec. \ref{sec:two-classes}.
\begin{conj}\label{conj:unicidad} The map $\mathbf{f}$ has a unique straight line of fixed points which is of the form \eqref{eq:straight-line}.   
\end{conj}
An example for $K\!=\!3$ can be seen in Fig. \ref{fig:ej-3-3D}.

\subsection{Jacobian Matrix of the SUCPA-map}

The Jacobian matrix of a system is an essential feature to study, for example, the local dynamics around the fixed points of a map \cite{Devaney,Wiggins}. 
\smallskip

The Jacobian matrix of the map $\mathbf{f}$ evaluated at $\bsy{\beta}$, $\mathbf{J}(\bsy{\beta})\!\in\!\mathbb{R}^{K\times K}$, has elements, $J_{k,\ell}(\bsy{\beta})$, which can be computed as:
\begin{equation}\label{matriz-J}
J_{k,\ell}(\bsy{\beta})=\frac{\partial f_k}{\partial \beta_\ell}(\bsy{\beta})=\frac{\ds \sum_{i=1}^N\frac{P_{i,k}P_{i,\ell}\,e^{\beta_\ell}}{\big(\sum_{j=1}^KP_{i,j}\,e^{\beta_j}\big)^2}} {\ds	\sum_{i=1}^N\frac{P_{i,k}}{\sum_{j=1}^KP_{i,j}\,e^{\beta_j}}}\, ,\quad k,\ell\!=\!1,\hdots,K
\end{equation}
In particular, $\mathbf{J}(\bsy{\beta})$ turns to be a {\it regular transition probability matrix} for all $\bsy{\beta}\!\in\!\mathbb{R}^K$, and so, it has the following properties \cite{meyer00}:
\begin{lemma}\label{transition-matrix}
The $\mathbf{J}(\bsy{\beta})$ matrix defined by the elements of \eqref{matriz-J} verifies:
\begin{itemize}
\item[(i)] $J_{k,\ell}(\bsy{\beta})\!>\!0$, $\forall\;k,\ell$.
\item[(ii)]  Each row of the matrix adds up to 1, i.e.: \begin{equation}\label{eq:suma-filas}
\sum_{\ell=1}^KJ_{k,\ell}(\bsy{\beta})\!=\!1,\  \forall\;1\leq k\leq K\end{equation}
This condition necessarily means that $\mathbf{J}(\bsy{\beta})$ has an eigenvalue $\mu\!=\!1$  with $\mathbf{1}\!=\![1,\hdots,1]$ as an associated eigenvector.
\item[(iii)] $\mu\!=\!1$ has multiplicity one, and all others eigenvalues verify $|\mu|\!<\!1$.
\end{itemize}
\end{lemma}
\begin{proof}
Items (i) and (ii) are straightforward to find out. Item (iii) is a consequence of Perron's Theorem \cite[Chapter 8]{meyer00}.  
\end{proof}

\subsection{Complementary definitions and properties}
To deepen into some details, we add other necessary concepts and results.
\begin{defn}
The {\it forward orbit of $\bsy{\beta}\!\in\!\mathbb{R}^K$} is the set: 
\begin{equation}
O^+(\bsy{\beta})=\big\{\mathbf{f}^{(t)}(\bsy{\beta}), t\!\geq\!0\big\}    
\end{equation}
\end{defn}
$O^+(\bsy{\beta})$ is the set of all points that the algorithm will pass through {when starting at $\bsy{\beta}$.} 

\begin{defn}
${\cal X}\!\subset\!\mathbb{R}^K$ is an {\it invariant set} if $\mathbf{f}({\cal X})\!\subset\!{\cal X}$. If $\mathbf{f}({\cal X})\!=\!{\cal X}$, ${\cal X}$ is called {\it strongly invariant} (or s-invariant for short).
\end{defn}
This means that all forward orbits of points in ${\cal X}$ lie in ${\cal X}$. In particular, any set of fixed points is s-invariant. Also, for all $\bsy{\beta}$, the set $\omega(\bsy{\beta})$ is a s-invariant set.

\begin{defn} A fixed point $\bsy{\beta^\ast}$ of a map is called {\it unstable}, if at least one eigenvalue of $\mathbf{J}(\bsy{\beta^\ast})$ has modulus greater than one. Otherwise, it is called {\it stable}. If all eigenvalues have modulus less than one, it is called {\it asymptotically stable}.   
\end{defn}

 

Let $\Delta_k(t)=\beta_k^{[t+1]}\!-\!\beta_k^{[t]}$ be the increment of the $k^{\mbox{th}}$ component of an orbit. 
The following interesting lemma is related to them.

\begin{lemma}\label{lem:aux} For
$\Delta_k(t)$ the following condition holds:
\begin{equation}
\sum_{k=1}^KN_ke^{-\Delta_k(t)}=N
\end{equation}
\end{lemma}
\begin{proof}
By Eq. \eqref{eq:algoritmo}, it is not hard to see that each $\Delta_k(t)$ can be written as:
\begin{equation}\label{eq:Deltadef}
\Delta_k(t)=-\log\left(\frac{1}{N_k}\sum_{i=1}^N\frac{P_{i,k}e^{\beta_k^{[t]}}}{\sum_{j=1}^KP_{i,j}e^{\beta_j^{[t]}}}\right)    
\end{equation}
taking exponential functions on both sides of the equation:
\begin{equation}
N_ke^{-\Delta_k(t)}=\sum_{i=1}^N\frac{P_{i,k}e^{\beta_k^{[t]}}}{\sum_{j=1}^KP_{i,j}e^{\beta_j^{[t]}}}
\end{equation}
and finally adding all the terms over $k$, the proof is done.
\end{proof}
Lemma \ref{lem:aux} shows that, in some sense, the results at every step of the algorithm values are ``balanced'', since, if for some $k$, $\beta_k$ tends to increase with the iterations, then there must be, at least, another $k'$, for which $\beta_{k'}$ tend to decrease (and vice versa).


\section{The case of two classes}\label{sec:two-classes}

For $K\!=\!2$, $\bsy{\beta}\!=\![\beta_1,\beta_2]$, the SUCPA algorithm \eqref{eq:algoritmo}, and the map $\mathbf{f}\!=\![f_1,f_2]\!:\mathbb{R}^2\!\to\!\mathbb{R}^2$, become into:
\begin{equation}\label{eq:algoritmo-2D}
\left\{
\begin{array}{l}
f_1(\bsy{\beta}^{[t]})=\beta_1^{[t+1]}=\displaystyle-\log\left(\frac{1}{N_1}\sum_{i=1}^N\frac{P_{i,1}}{P_{i,1}\,e^{\beta_1^{[t]}}\!+\!P_{i,2}\,e^{\beta_2^{[t]}}}\right)\\[4mm]
f_2(\bsy{\beta}^{[t]})=\beta_2^{[t+1]}=\displaystyle-\log\left(\frac{1}{N_2}\sum_{i=1}^N\frac{P_{i,2}}{P_{i,1}\,e^{\beta_1^{[t]}}\!+\!P_{i,2}\,e^{\beta_2^{[t]}}}\right)
\end{array}\right.
\end{equation}
where $N_1\!+\!N_2\!=\!N$, $P_{i,k}\!>\!0$, $k\!=\!1,2$ and $P_{i,1}\!+\!P_{i,2}\!=\!1$, for all $i\!=\!1,\hdots,N$. 

Regarding the increments, Lemma \ref{lem:aux} remains:
\begin{equation}\label{eq:clave}
N_1\,e^{-\Delta_1(t)}\!+\!N_2\,e^{-\Delta_2(t)}=N
\end{equation}

In addition, for $K\!=\!2$ it is possible to derive an additional property. Once it is known that the algorithm must be convergent, it is necessary that $\Delta_1(t)$ and $\Delta_2(t)$ go to zero when $t\!\to\!\infty$. From Eq. \eqref{eq:clave}, and applying the L'Hopital's rule, it can be obtained the direction of convergence:
\begin{equation}
  \lim_{t\to+\infty} \frac{\Delta_2(t)}{\Delta_1(t)}= \lim_{\Delta_1\to 0} \frac{-\log\bigg(\dfrac{N\!-\!N_1e^{-\Delta_1}}{N_2}\bigg)}{\Delta_1}=\lim_{\Delta_1\to 0} \frac{-N_2}{(N\!-\!N_1e^{-\Delta_1})}\,\frac{N_1}{N_2}\,e^{-\Delta_1}=-\frac{N_1}{N_2}
\end{equation}


This means that $O^+(\bsy{\beta}^{[0]})\!\subset\!\mathbb{R}^2$ becomes tangent to a straight line of slope $-N_1/N_2$, and this is true for every i.c. $\bsy{\beta}^{[0]}$.
\smallskip

In order to obtain the fixed points of Eq. \ref{eq:algoritmo-2D}, for $x\!\in\!\mathbb{R}$, let ${\cal S}_x$ be the straight line with slope 1 and crossing the $y$-axis at $x$ value:
\begin{equation}\label{eq:recta-x}
   {\cal S}_x=\{\bsy{\beta}\!=\![\lambda,\lambda\!+\!x],\ \lambda\!\in\!\mathbb{R}\} 
\end{equation}
As the straight line of fixed point to be found is of this type, it will be enough to determine only its intercept $x$.

\subsection{Main results}
The results will be splitted up into those related to convergence and those related to the Jacobian matrix.

\subsubsection{About the convergence}
As said before, Conjectures \ref{conj:convergencia} and \ref{conj:unicidad} can be proven when the number of classes $K\!=\!2$. The main result can be stated as follow: {\it There exists a unique straight line of slope 1 built up by fixed points, to which the algorithm converges, for all initialization.} In terms of dynamical systems, this is equivalent to state the following theorems.
\begin{thm}\label{main-result-1} 
Let $\mathbf{f}$ be the map defined in \eqref{eq:algoritmo-2D}. Then, there exists a unique $b\!\in\!\mathbb{R}$ , such that the fixed point set of $\mathbf{f}$ is $\mathcal{S}_b$.
\end{thm}

Note that all fixed points are not-isolated, and ${\cal S}_b$ is an unbounded set.

\begin{thm}\label{main-result-2} 
Let $\mathbf{f}$ be the map defined in \eqref{eq:algoritmo-2D}. For all i.c. $\boldsymbol{\beta}^{[0]}$, there exist a fixed point, $\boldsymbol{\beta^\ast}\!\!\in{\cal S}_b$ 
($\bsy{\beta^\ast}(\bsy{\beta}^{[0]})$), such that $\omega(\boldsymbol{\beta}^{[0]})\!=\!\{\boldsymbol{\beta^\ast}\}$.  
\end{thm}
Both proofs are developed in Section \ref{proof-main-results}. 


\subsubsection{About the Jacobian matrix}
As said before, the eigenvalues and eigenvectors of the Jacobian matrix of a map, evaluated at a fixed point, account for the local dynamics of the system around that point. A fixed point $\bsy{\beta^\ast}$ of a map $\mathbf{f}$ is call {\it hyperbolic}, if the Jacobian matrix, $\mathbf{J}_{\mathbf{f}}(\bsy{\beta^\ast})$, has no eigenvalues, $\mu$ in the unit circle. Otherwise, it is called {\it non-hyperbolic}. The invariant spaces, $E^s,E^u$ and $E^c$, of $\mathbf{J}_{\mathbf{f}}(\bsy{\beta^\ast})$, associated to $|\mu|\!>\!1$, 
 $|\mu|\!<\!1$ and $|\mu|\!=\!1$, respectively, correspond to the {\it local} stable, the {\it local} unstable and the {\it local } 
  ``central'' directions, respectively \cite{Devaney,Wiggins}.

By Lemma \ref{transition-matrix}, the two eigenvalues of $\mathbf{J}_{\mathbf{f}}(\bsy{\beta^\ast})$ are, 1 with associated eigenvector $\mathbf{1}$, and the other, $\mu$, must verify $|\mu|\!<\!1$, so, the fixed points are all of non-hyperbolic type. The eigenvalue $1$ is in accordance with the obtaining of non-isolated fixed points \cite{Jamieson}, and the eigenvector $\mathbf{1}$ is in accordance with the direction of the  fixed points straight line.
Therefore, $E^u\!=\!\emptyset$, $E^c\!=\!\mbox{span}\{\mathbf{1}\}$ and $E^s\!=\!\mbox{span}\{\mathbf{v}\}$, where $\mathbf{v}$ is the eigenvector associated to $\mu$ to be determined. The following lemma shows some features for the eigen-pair of the Jacobian matrix for the fixed points.

\begin{lemma}\label{lem:eigenpair}
Let $\bsy{\beta^\ast}\!\!\in\!\mathbb{R}^2$ a fixed point of $\mathbf{f}$. The eigen-pairs of $\mathbf{J}_{\mathbf{f}}(\bsy{\beta^\ast})$ are $(1,\mathbf{1})$ and $(\mu,\mathbf{v})$ where:
\begin{itemize}
    \item[(i)] The value of $\mu$ and $\mathbf{v}$ are the same for all fixed point $\bsy{\beta^\ast}$
    \item[(ii)] $0\leq\mu<1$
    \item[(iii)] $\mathbf{v}=[N_2,-N_1]$
\end{itemize}
\end{lemma} 

The proof can be seen in Section \ref{proof-main-results}. Fig. \ref{fig:local-dynamics} illustrate the phase portrait of the local dynamics near ${\cal S}_b$. The eigenvalues and eigenvectors of $\mathbf{J}(\bsy{\beta^\ast})$ account for the local dynamics, so every point $\bsy{\beta^\ast}\in\mathcal{S}_b$ is stable, which means that there exists a neighborhood of $\bsy{\beta^\ast}$, say $B(\bsy{\beta^\ast}\!,\varepsilon)$, for some $\varepsilon\!>\!0$, such that every orbit that enters it, stays there. Besides, as the $\varepsilon$ value is {\it the same} for all $\bsy{\beta^\ast}$ (because the Jacobian matrix is), it turns out that the set:
\begin{equation}
{\cal S}_b(\varepsilon)=\!\bigcup_{\bsy{\beta^\ast}\in\,{\cal S}_b} \!B(\bsy{\beta^\ast}\!,\varepsilon)    
\end{equation}
is an invariant set for the map $\mathbf{f}$. The set ${\cal S}_b(\varepsilon)$ is nothing but the strip of $2\varepsilon$-width around ${\cal S}_b$, and the dynamics within it will be characterize by the sign of $\mu$ (see Fig. \ref{fig:local-dynamics}).
\begin{figure}[ht]
  \centering\includegraphics[angle=-90,scale=0.3]{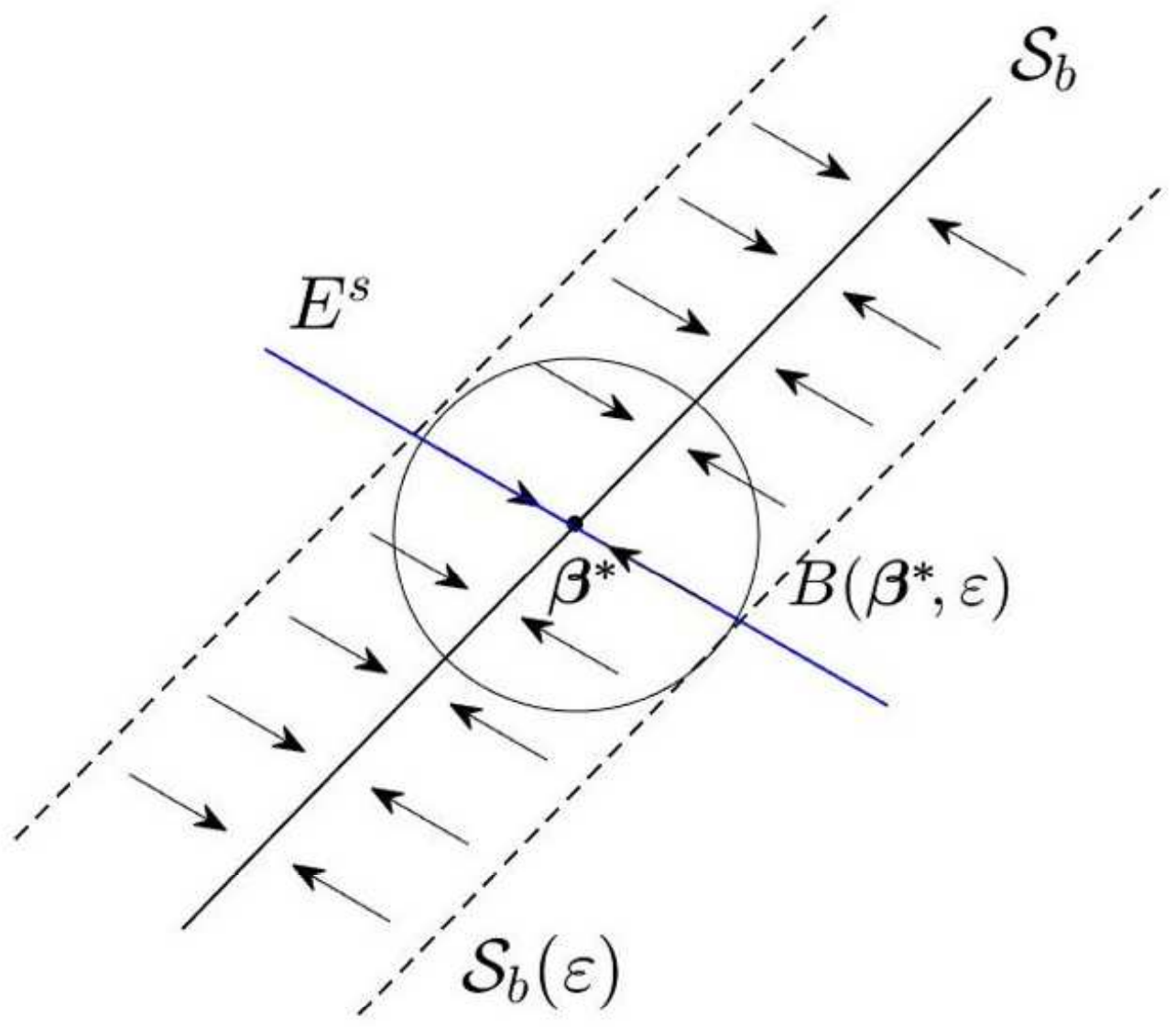}
  \caption{Schematic representation of local stability at a fixed point $\bsy{\beta^\ast}\!\!\in\!{\cal S}_b$ within a ball of radius $\varepsilon$ centered at $\bsy{\beta^\ast}$, $B(\bsy{\beta^\ast}\!,\varepsilon)$. The stability direction $E^s$ (blue line). The direction of $E^c$ matches that of the ${\cal S}_b$. The local dynamics is a replica for each $\bsy{\beta^\ast}\!\!\in{\cal S}_b$ within the $2\varepsilon$-width strip  ${\cal S}_b(\varepsilon)$}
  \label{fig:local-dynamics}
\end{figure}

In order to prove the main results, in the next section we introduce some auxiliary functions useful in the technical details of the proofs.

\subsubsection{Auxiliary functions and its properties}\label{sec:auxiliar}
 
 Let $\alpha_1,\alpha_2\!:\!\mathbb{R}\!\to\!\mathbb{R}$ be the following function: 
\begin{equation}\label{alfas}
 \alpha_1(x)=\frac{1}{N_1}\sum\limits_{i=1}^N\frac{P_{i,1}}{P_{i,1}\!+\!P_{i,2}\,e^x}\qquad \alpha_2(x)=\frac{1}{N_2}\sum\limits_{i=1}^N\frac{P_{i,2}}{P_{i,1}\!+\!P_{i,2}\,e^x}
\end{equation}
with $P_{i,k}\!>\!0$, $k\!=\!1,2$, $P_{i,1}\!+\!P_{i,2}\!=\!1$, $i\!=\!1,\hdots,N$. Functions $\alpha_k(x)$ are clearly positive and continuous derivable. Also is easy to see they satisfy the following relation, for $x\!\in\!\mathbb{R}$:
\begin{equation}\label{relacion-alfas}
N_1\alpha_1(x)+N_2\,e^x\alpha_2(x)=N
\end{equation}

\begin{lemma}\label{monotonía-alfas}
$\alpha_k(x)$ are monotone decreasing on $x$, $k\!=\!1,2$.
\end{lemma}

\begin{proof} Indeed, for all $x\!\in\!\mathbb{R}$:
\begin{equation}\label{alfas-derivadas}
 \alpha'_1(x)\!=\!\frac{1}{N_1}\sum\limits_{i=1}^N\frac{-P_{i,1}P_{i,2}\,e^x}{\left(P_{i,1}\!+\!P_{i,2}\,e^x\right)^2}<0\qquad \alpha'_2(x)\!=\!\frac{1}{N_2}\sum\limits_{i=1}^N\frac{-P_{i,2}^{\,2}\,e^x}{\left(P_{i,1}\!+\!P_{i,2}\,e^x\right)^2}<0
\end{equation} 
\end{proof}

\subsection{Proofs of the main results}\label{proof-main-results}
Bellow are the proofs of Theo. \ref{main-result-1}, Lemma \ref{lem:eigenpair} and  Theo. \ref{main-result-2}. 

The first theorem is about the existence and uniqueness of a straight line, $\mathcal{S}_x$, of fixed points of $\mathbf{f}$, for some $x$. 

\begin{proof} {\it Theo. \ref{main-result-1}}
The strategy is to prove that there exists a unique value $b\!\in\!\mathbb{R}$, such that $\bsy{\beta^\ast}\!=\![0,b]$ is a fixed point of $\mathbf{f}$. Thus, by virtue of Corollary \ref{recta-S}, the straight line ${\cal S}_b$  is the desired line (see Eq. \eqref{eq:recta-x}). For $[0,b]$ to be a fixed point of $\mathbf{f}$, it has to verify, $\mathbf{f}\big([0,b]\big)\!=\![0,b]$ in Eq. \eqref{eq:algoritmo-2D}. 
\begin{equation} \label{eq:puntos-fijos}
\left\{
\begin{array}{l}
f_1\big([0,b]\big)=\displaystyle-\log\left(\frac{1}{N_1}\sum_{i=1}^N\frac{P_{i,1}}{P_{i,1}\,e^0\!+\!P_{i,2}\,e^b}\right)=0\\[4mm]
f_2\big([0,b]\big)=\displaystyle-\log\left(\frac{1}{N_2}\sum_{i=1}^N\frac{P_{i,2}}{P_{i,1}\,e^0\!+\!P_{i,2}\,e^b}\right)=b
\end{array}\right.
\end{equation}

Using the auxiliary functions $\alpha_k$ defined in Eq. \eqref{alfas}, this is equivalent to prove that there exists a unique $b\!\in\!\mathbb{R}$, such that $\alpha_1(b)\!=\!1$ and $\alpha_2(b)\!=\!e^{-b}$.

Indeed, taking limits for $\alpha_1(x)$ when $x\!\to\!-\infty$ and  $x\!\to\!+\infty$ respectively, it is obtained: 
\begin{align}
\lim_{x\to-\infty}\alpha_1(x)&=\lim_{x\to-\infty}\frac{1}{N_1}\sum\limits_{i=1}^N\frac{P_{i,1}}{P_{i,1}\!+\!P_{i,2}\,e^x}=\frac{N}{N_1}>1 \\
\lim_{x\to+\infty}\alpha_1(x)&=\lim_{x\to+\infty}\frac{1}{N_2}\sum\limits_{i=1}^N\frac{P_{i,1}}{P_{i,1}\!+\!P_{i,2}\,e^x}=0<1 \label{eq:alpha<1} 
\end{align}
The existence follows from the mean value theorem and the uniqueness from the $\alpha_1$ strict monotony. The proof of $\alpha_2(b)\!=\!e^{-b}$ comes straight out from \eqref{relacion-alfas}.    
\end{proof}

Note that $b$ is obtained implicitly, and only can be found by numerical methods.
\smallskip

The second theorem is about the global stability of the map $\mathbf{f}$. The strategy is to prove that $\mathcal{S}_b$ is a {\it global attractor} set for $\mathbf{f}$.
\subsubsection{About the intercepts}

For every point $\boldsymbol{\beta}\!=\![\beta_1,\beta_2]\in\mathbb{R}^2$ there is a single straight line, $\mathcal{S}_x$, of slope 1 and intercept, $x\!=\!\beta_2\!-\!\beta_1$, that passes through it. This section is devoted to show that every ${\cal S}_x$ monotonously approaches to $\mathcal{S}_b$, where $b$ is the intercept of the fixed points line, previously obtained. Subtracting the equations in \eqref{eq:algoritmo-2D} and using the auxiliary functions $\alpha_k$ of \ref{alfas}, it is obtained:
\begin{equation}
    \beta_2^{[t+1]}\!-\!\beta_1^{[t+1]}=\log\left(\frac{\alpha_1\big(\beta_2^{[t]}\!-\!\beta_1^{[t]}\big)}{\alpha_2\big(\beta_2^{[t]}\!-\!\beta_1^{[t]}\big)}\right)
\end{equation}
Thus, for each $x\!\in\!\mathbb{R}$, the function:
\begin{equation}\label{eq:updates}
\phi(x)=\log\left(\frac{\alpha_1(x)}{\alpha_2(x)}\right)    \end{equation} 
defines the update of the $\mathcal{S}_x$ intercepts, according to the iterations of $\mathbf{f}$. In particular, for the $\mathcal{S}_b$ line of fixed points, this relationship implies:
\begin{equation}\label{eq:balphas}
    b=\phi(b)=\log\left(\frac{\alpha_1(b)}{\alpha_2(b)}\right)
\end{equation}
The following lemma shows that the intercepts $x$ monotonously approach to $b$.

\begin{lemma}\label{desigualdades-alfas}\
\begin{itemize}
\item[(i)] For $x\!>\!b$:\quad
$b\leq\phi(x)<x$
\item[(ii)] For $x\!<\!b$:\quad 
$x<\phi(x)\leq b$
\end{itemize}
\end{lemma}
\begin{proof}
The right inequality in Lemma \ref{desigualdades-alfas} (i) is fairly straightforward to prove. By Lemma \ref{monotonía-alfas} and \eqref{eq:alpha<1}, we have $\alpha_1(x)\!<\!1\!=\!\alpha_1(b)$, and then, \eqref{relacion-alfas} implies $e^x\alpha_2(x)\!>\!1$. Thus, $\alpha_1(x)\!<\!e^x\alpha_2(x)$, which is equivalent to the statement ($\alpha_2(x)>0$). 
\smallskip

To obtain the left inequality, we must ``split hair''. It is equivalent to prove:
\begin{equation}\label{desigualdad-ii}
\frac{\alpha_2(x)}{\alpha_1(x)}\le e^{-b}
\end{equation}
Let $g$ be a function on $x$ defined as:
\begin{equation}
   g(x)={e^{\phi(x)}}=\frac{\alpha_2(x)}{\alpha_1(x)} 
\end{equation}
Note that $g(b)\!=\!e^{-b}$, via \eqref{eq:balphas}. Thus, to obtain \eqref{desigualdad-ii}, all that remains is to prove that $g$ is a monotone non-increasing function, i.e., $g'(x)\!\le\!0$, for all $x\!>\!b$.
This is equivalent to show that:
\begin{equation}\label{eq:desig-fundamental}
 N_1N_2\big(\alpha_2'(x)\alpha_1(x)-\alpha_1'(x)\alpha_2(x)\big)\le 0   
\end{equation}
By using the expressions of $\alpha_k$ and $\alpha_k'$ in \ref{alfas} and \ref{alfas-derivadas}:
\begin{align}\label{términos-simétricos}
&N_1N_2\,\alpha_2'(x)\alpha_1(x)-N_1N_2\,\alpha_1'(x)\alpha_2(x)=\\
&=\left[\sum_{i=1}^N\frac{-P_{i,2}^{\,2}\,e^x}{\big(P_{i,1}\!+\!P_{i,2}\,e^x\big)^2}\right]\!\left[\sum_{j=1}^N\frac{P_{j,1}}{P_{j,1}\!+\!P_{j,2}\,e^x}\right]-\!\left[\sum_{i=1}^N\frac{-P_{i,1}P_{i,2}\,e^x}{\big(P_{i,1}\!+\!P_{i,2}\,e^x\big)^2}\right]\!\left[\sum_{j=1}^N\frac{P_{j,2}}{P_{j,1}\!+\!P_{j,2}\,e^x}\right]\\
&=\sum_{j=1}^N\sum_{i=1}^N\frac{- P_{i,2}^{\,2}\,e^xP_{j,1}+P_{i,1}P_{i,2}\,e^xP_{j,2}}{\big(P_{i,1}\!+\!P_{i,2}\,e^x\big)^2\big(P_{j,1}\!+\!P_{j,2}\,e^x\big)}=e^x\sum_{j=1}^N\sum_{i=1}^N\frac{P_{i,2}\,\big(P_{i,1}P_{j,2}\!-\!P_{i,2}P_{j,1}\big)}{\big(P_{i,1}\!+\!P_{i,2}\,e^x\big)^2\big(P_{j,1}\!+\!P_{j,2}\,e^x\big)}\\
&=e^x\sum_{j=1}^N\sum_{i=1}^N\frac{ P_{i,2}\,\big(P_{i,1}P_{j,2}\!-\!P_{i,2}P_{j,1}\big)\big(P_{j,1}\!+\!P_{j,2}\,e^x\big)}{\big(P_{i,1}\!+\!P_{i,2}\,e^x\big)^2\big(P_{j,1}\!+\!P_{j,2}\,e^x\big)^2}\label{eq:numeradores}\\
&=e^x\sum_{j=1}^N\sum_{i=1}^N\frac{-\big(P_{i,1}P_{j,2}\!-\!P_{i,2}P_{j,1}\big)^2}{\big(P_{i,1}\!+\!P_{i,2}\,e^x\big)^2\big(P_{j,1}\!+\!P_{j,2}\,e^x\big)^2}\leq 0
\end{align}
The last equality comes from the fact that in Eq. \eqref{eq:numeradores}, the terms with same subscripts ($i\!=\!j$), equal zero; and for $i\!\neq\!j$, by lumping together and adding two symmetric terms on $i$ and $j$, the numerators become:
\begin{align}
 &P_{i,2}\,\big(P_{i,1}P_{j,2}\!-\!P_{i,2}P_{j,1}\big)\big(P_{j,1}\!+\!P_{j,2}\,e^x\big)+ P_{j,2}\,\big(P_{j,1}P_{i,2}\!-\!P_{j,2}P_{i,1}\big)\big(P_{i,1}\!+\!P_{i,2}\,e^x\big)=\\
 & = \big(P_{i,1}P_{j,2}\!-\!P_{i,2}P_{j,1}\big)\big[P_{i,2}\,\big(P_{j,1}\!+\!P_{j,2}\,e^x\big)\!-\!P_{j,2}\,\big(P_{i,1}\!+\!P_{i,2}\,e^x\big)\big]\\[0.5mm]
 & =\big(P_{i,1}P_{j,2}\!-\!P_{i,2}P_{j,1}\big)\big[P_{i,2}P_{j,1}+P_{i,2}P_{j,2}\,e^x\!-\!P_{j,2}P_{i,1}\!-\!P_{j,2}P_{i,2}\,e^x\big]\\[0.5mm]
 &= -\big(P_{i,1}P_{j,2}\!-\!P_{i,2}P_{j,1}\big)^2
\end{align}

Then, \eqref{eq:desig-fundamental} holds, $g^\prime(x)\leq0$ and $g(x)\leq e^{-b}$. 
In an analogous way, it can be proven the same inequalities corresponding to Lemma \ref{desigualdades-alfas} (ii), for $x\!<\!b$.
\end{proof}
\begin{figure}[ht]
  \centering\includegraphics[scale=0.4]{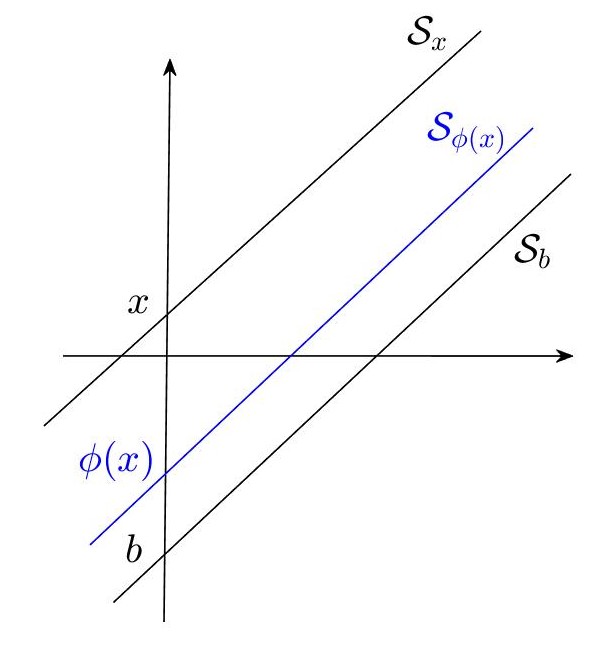}
  \caption{Any unitary slope straight line ${\cal S}_x$ above the line of fixed points, ${\cal S}_b$, is mapped onto another one, $\mathbf{f}({\cal S}_x)={\cal S}_{\phi(x)}$ (in blue), placed between them, but strictly above ${\cal S}_x$}
  \label{fig:rectas}
\end{figure}

The geometrical interpretation of this lemma is the following. Let ${\cal S}_b^{\uparrow}$ and ${\cal S}_b^{\downarrow}$ be the semi-planes located above and below ${\cal S}_b$, respectively, then they are invariant sets for $\mathbf{f}$.
\begin{coro}\label{co:semi-planes} If $\bsy{\beta}\!\in\!{\cal S}_b^{\uparrow}$, then $O^+(\bsy{\beta})\!\in\!{\cal S}_b^{\uparrow}$, and if $\bsy{\beta}\!\in\!{\cal S}_b^{\downarrow}$, then $O^+(\bsy{\beta})\!\in\!{\cal S}_b^{\downarrow}$.    
\end{coro}

Moreover, if $x\!>\!b$ (i.e. ${\cal S}_x$ is placed above ${\cal S}_b$), then, Lemma \ref{desigualdades-alfas} states that ${\cal S}_{\phi(x)}$ will be placed between ${\cal S}_x$ and ${\cal S}_b$, but {\it strictly} under the first one. This feature can be seen in the Fig. \ref{fig:rectas}. 
\begin{proof} {\it Lemma \ref{lem:eigenpair}} 
Theorem \ref{main-result-1} states that the fixed point of $\mathbf{f}$ are those $\bsy{\beta^\ast}\!=\![\lambda,\lambda\!+\!b]$, $\lambda\!\in\!\mathbb{R}$, with $b$ the only value that verifies $\alpha_1(b)\!=\!1$ and $\alpha_2(b)\!=\!e^{-b}$. By Eq.\eqref{matriz-J} the Jacobian matrix of $\mathbf{f}$ at these points results:
\begin{equation}\label{eq:matriz-J-2D} \mathbf{J}_b=\mathbf{J}_{\mathbf{f}}\big([\lambda,\lambda\!+\!b]\big)= 
 \begin{bmatrix}
 \ds \frac{1}{N_1}\sum_{i=1}^N\frac{P_{i,1}^{\,2}}{(P_{i,1}\!+\!P_{i,2}\,e^b)^2} & \ds \frac{1}{N_1}\sum_{i=1}^N\frac{P_{i,1}P_{i,2}\,e^b}{(P_{i,1}\!+\!P_{i,2}\,e^b)^2} \\[4mm]
 \ds \frac{1}{N_2}\sum_{i=1}^N\frac{P_{i,1}P_{i,2}\,e^b}{(P_{i,1}\!+\!P_{i,2}\,e^b)^2} & \ds \frac{1}{N_2}\sum_{i=1}^N\frac{P_{i,2}^{\,2}(e^b)^2}{(P_{i,1}\!+\!P_{i,2}\,e^b)^2}   
 \end{bmatrix}
\end{equation}
Note that $\mathbf{J}_b$ does not depend on $\lambda$, i.e., it is the same for all $\bsy{\beta^\ast}\!\in\!{\cal S}_b$, and so, the same the eigenvalues and eigenvectors. As a consequence, the local dynamics at one fixed point results in an exact replica at every $\bsy{\beta^\ast}$. This fact proves item $(i)$.

In order to obtain both, the other eigenvalue $\mu$ and its eigenvector $\mathbf{v}$, the matrix entries in Eq. \eqref{eq:matriz-J-2D} are written  in simplified form by means of the auxiliary functions $\alpha_k, k\!=\!1,2$. Using Eqs. \eqref{eq:suma-filas} and \eqref{alfas-derivadas}, the matrix in \eqref{eq:matriz-J-2D}, becomes:
\begin{equation}\label{eq:matriz-J-2D-alfas} 
\mathbf{J}_b=\mathbf{J}_{\mathbf{f}}\big([\lambda,\lambda\!+\!b]\big)= 
 \begin{bmatrix}
 1\!+\!\alpha_1'(b) & -\alpha'_1(b) \\[2mm]
1\!+\!e^b\alpha_2'(b) & -e^b\alpha'_2(b)   
 \end{bmatrix}
\end{equation}

The trace of $\mathbf{J}_b$ equals the sum of its eigenvalues, so:
\begin{equation}\label{eq:mu}
\mu=\alpha_1'(b)-e^b\alpha_2'(b)    
\end{equation}
Let $h(x)$ be another auxiliary function on $x$ defined as $h(x)\!=\!\alpha_2'(x)\alpha_1(x)-\alpha_1'(x)\alpha_2(x)$. By virtue of \eqref{eq:desig-fundamental}, we have that $h(x)\!\le\!0$ for all $x\!\in\!\mathbb{R}$. In particular, for $x\!=\!b$, and using that $\alpha_1(b)\!=\!1$ and $\alpha_2(b)\!=\!e^{-b}$, it is possible to obtain the sign of $\mu$:
\begin{align}
h(b)=\alpha'_2(b)\alpha_1(b)-\alpha'_1(b)\alpha_2(b)&\le 0  \\
\alpha'_2(b)-\alpha'_1(b)\,e^{-b} &\le 0\\
\alpha'_2(b)\,e^b-\alpha'_1(b) &\le 0\\
-\mu &\le 0
\end{align}
Therefore, item (ii) is proven. As a consequence of $\mu$ being non-negative, any trajectory that approaches ${\cal S}_b$ from one side does not cross to the other side.

For item (iii), on the one hand, the anti-diagonal entries of $\mathbf{J}_b$ in \eqref{eq:matriz-J-2D} satisfy: 
\begin{equation}\label{eq:antidiagonal}
N_1J_{1,2}\!=\!N_2J_{2,1}
\end{equation}
then, the entries of $\mathbf{J}_b$ in \eqref{eq:matriz-J-2D-alfas} meet:
\begin{equation}\label{eq:alfa'_1(b)}
-N_1\alpha'_1(b)=N_2\big(1\!+\!e^b\alpha_2'(b)\big) 
\end{equation}
On the other hand, using the expression of $\mu$ in \eqref{eq:mu} we have:
\begin{align}
E^s=\mbox{Nul}\big(\mathbf{J}_b\!-\!\mu\,\mathbf{I}\big)&= \mbox{Nul}\left(\begin{bmatrix}
 1\!+\!e^b\alpha_2'(b) & -\alpha'_1(b) \\[2mm]
1\!+\!e^b\alpha_2'(b) & -\alpha'_1(b)  
 \end{bmatrix}\right)\\[2mm]
 &=\mbox{span}\left\{ [\alpha'_1(b), 1\!+\!e^b\alpha_2'(b)]\right\}\\
 &=\mbox{span}\big\{[N_2,-N_1]\big\}
\end{align}  
\end{proof}

Then, $E^s\!=\!\mbox{span}\left\{\mathbf{v}\!=\![N_2,-N_1]\right\}$. Note that, in case that $N_1\!=\!N_2$, $E^s$ and $E^c\!=\!\mathcal{S}_b$ become orthogonal.

\smallskip

The strategy for proving Theo. \ref{main-result-2} depends largely on the geometrical interpretation of Lemma \ref{lem:singletransition} which, in the case of two dimensions, means that any unitary straight line is mapped by $\mathbf{f}$ onto another unitary straight line. The idea is to prove that any ${\cal S}_x$ gets closer and closer to the fixed line ${\cal S}_b$ with successive iterations of the map. As they are parallel lines, it is enough to study only the sequence given by the intercepts $x$. First, the following auxiliary lemmas will be useful.

\begin{lemma} For all $x\!\in\!\mathbb{R}$, $\mathbf{f}^{(t)}({\cal S}_x)\to {\cal S}_b$, as $t\to +\infty$    
\end{lemma}
\begin{proof} 
Let $x\!\in\!\mathbb{R}$, and $\bsy{\beta}\!=\![0,x]$. Evaluating $\mathbf{f}$ in $\beta$ in Eq. \eqref{eq:algoritmo-2D}, and using functions $\alpha_k$, we have that:
\begin{equation}
\mathbf{f}\big([0,x]\big)\!=\![-\log\alpha_1(x),-\log\alpha_2(x)]
\end{equation}
so, by Lemma \ref{lem:singletransition}:
\begin{equation}
\mathbf{f}\big([0,x]\!+\!\bsy{\lambda}\big)=\bsy{\lambda}+[-\log\alpha_1(x),-\log\alpha_2(x)]
\end{equation}
which, in terms of Eq. \eqref{eq:recta-x}, means that $\mathbf{f}({\cal S}_x)\!=\!{\cal S}_{\phi(x)}$, where $\phi(x)$ is the function defined in Eq. \eqref{eq:updates}, i.e., the unitary slope straight line with $x$ intercept, is mapped to the unitary slope straight line with $\phi(x)$ intercept (see Fig. \ref{fig:rectas}). Recursively, by applying Lemma \ref{lem:eigenpair} to $\phi(x)\!>\!b$, we have $\mathbf{f}^{(2)}({\cal S}_x)\!=\!{\cal S}_{\phi^{2}(x)}$, where $\phi^{2}\!=\!\phi\circ\phi$, and so on. In other words, for $x\!>\!b$, the sequence of intercepts, $\{\phi^t(x)\}_{t\ge 1}$, is monotone decreasing and:
\begin{equation}
    \lim_{t\to+\infty}\phi^t(x)\!=\!b
\end{equation}
The analogous way, if $x\!<\!b$, ${\cal S}_{\phi(x)}$ will be placed between ${\cal S}_b$ and ${\cal S}_x$, but {\it strictly} above the latter. Following the previous arguments, the succession $\{\phi^t(x)\}_{t\ge 1}$ is monotone increasing and so, it converges to $b$.
\end{proof}
As a natural consequence of this lemma, it is possible to say that the orbit of any initial condition approaches ${\cal S}_b$. 
\begin{coro} For any i.c. $\bsy{\beta}^{[0]}$, $O^+(\bsy{\beta}^{[0]})\to{\cal S}_b$, for $t\to +\infty$     
\end{coro}
\begin{proof} Indeed, any point, $\bsy{\beta}^{[0]}\!\in\!\mathbb{R}^2$, $\bsy{\beta}^{[0]}\!=\!\big[\beta_1^{[0]},\beta_2^{[0]}\big]$ belongs to ${\cal S}_x$, for $x\!=\!\beta_1^{[0]}\!-\!\beta_1^{[0]}$.    
\end{proof}

\begin{proof} {\it Teo. \ref{main-result-2}}

The latest results show that the line ${\cal S}_b$ is a global asymptotically stable set for $\mathbf{f}$. However, in order to demonstrate the convergence of SUCPA algorithm, it is still necessary to show that the orbit of any i.c. does not approach {\it asymptotically} to ${\cal S}_b$ but, in fact, reaches a point $\bsy{\beta^\ast}$ on it. Although it is already evident that the asymptotic behavior is ruled out by the local dynamics illustrated in Fig. \ref{fig:local-dynamics}, we will also give rigorous argument on this point.

Let $\bsy{\beta}^{[t]}$ and $\bsy{\beta}^{[t+1]}$ be two consecutive points of $O^+(\bsy{\beta}^{[0]})$ and $\Delta_k(t)\!=\!\beta_k^{[t+1]}\!-\!\beta_k^{[t]}$, $k\!=\!1,2$, their increments. 
Then, the slope of the straight line through $\bsy{\beta}^{[t]}$ and $\bsy{\beta}^{[t+1]}$, is $\frac{\Delta_2(t)}{\Delta_1(t)}$.

Proceeding by {\it reductio ad absurdum}, if $O^+(\bsy{\beta}^{[0]})$ where asymptotic to ${\cal S}_b$, then the slopes should approach the unity: 
\begin{equation}\label{eq:incrementos}
\lim_{t\to+\infty}\frac{\Delta_2(t)}{\Delta_1(t)}=1
\end{equation}
Indeed, $O^+(\bsy{\beta}^{[0]})$ cannot cross over the straight line ${\cal S}_b$ due to Cor. \ref{co:semi-planes}, and so the ``oscillating'' asymptotic approach cannot take place. 
Finally, Eq. \eqref{eq:incrementos} is ruled out by Eq. \eqref{eq:clave}, given that the possibility for the increments are two: they are both null ($\Delta_k(t)\!=\!0$, and so, $\bsy{\beta}^{[t]}\!=\!\bsy{\beta}^{[t+1]}$), or they are of different sign, so
 $\dfrac{\Delta_2(t)}{\Delta_1(t)}\!<\!0$, for all $t$, and this is in contradiction to \eqref{eq:incrementos}. 

Thus, having excluded the asymptotic behavior, there must be a fixed point, $\bsy{\beta}^\ast\!\!\in\!{\cal S}_b$, to which the orbit converges, and then, the proof of Theo. \ref{main-result-2}, for $K\!=\!2$, is completed.
\end{proof}

\bigskip
\section{Numerical examples}\label{sec:numerical_examples}

In this section we present three real-world applications, showing the features proved for the $K\!=\!2$ cases and the conjectures made for the $K\!=\!3$ case. The first two examples correspond to language models and the third to an image classification application. Implementation details can be found in the repository of this work\footnote{\texttt{https://github.com/LautaroEst/sucpa-convergence}}.

\subsection{Application to Language models with K=2}
The first example corresponds to the calibration of the posterior probabilities of a language model when it is used as a zero-shot classifier. Specifically, a generative language model is a function that maps a string of characters $x$ (the prompt) to a probability distribution $P_{LM}(t|x)$ over a set of tokens $V\!=\!\{t_1,\ldots,t_{|V|}\}$ (the vocabulary), which represents the probability of the next token after the input string. Language models can be used as classifiers by computing the probability that the next token is the one that represents the class $y$ to be predicted. For instance, we can compute the probability that the words \textit{``positive''} or \textit{``negative''} appear after the sentence \textit{``Identify if the next review has a positive or negative connotation. Review: `I love this movie'. Sentiment:''}. This is an example of zero-shot classification because the prompt does not contain examples of correctly classified phrases. In order to obtain a probability distribution $P(y_k|x)$ for each class $y_k$, we compute the score $s_k\!=\!P_{LM}(w_k|x)$ with a predefined set of tokens (in the example, the words $w_1=${\it ``positive''} and $w_2=${\it ``negative''}) and then we normalized over all the scores:
\begin{equation}
    P(y_k|x) = \frac{s_k}{\sum_{k=1}^K s_k}
\end{equation}
This is a particularly good use-case scenario because these probabilities  produced by the language model are likely to be uncalibrated. For an extended explanation of this problem, see \cite{sucpa_arxiv}.

One of the families of classification problems are those known as \emph{polarity classification} and consists in determining if a given sentence has a positive or negative connotation. For this experiment, a subset of the SST-2 dataset \cite{sst2}, which contains examples of annotated movie reviews  was used. The number of positive and negative samples was $N_1\!=\!1729$ and $N_2\!=\!2271$, respectively. To obtain the class scores, the GPT-2 model \cite{gpt2} was used. 

Using the above values of $N_1$ and $N_2$ and the probabilities output by the model, a line ${\cal S}_b$ of fixed points with $b\!=\!-1.39726$ is obtained. 
Fig.~\ref{fig:ej-1-2D} shows this line and the orbits of five different i.c.: $\bsy{\beta}^{[0]}_1\!=\![0,2]$ (in red), $\bsy{\beta}^{[0]}_2\!=\![1.5,3.5]$ (blue), $\bsy{\beta}^{[0]}_3\!=\![3,0]$ (cyan), $\bsy{\beta}^{[0]}_4\!=\![4,-1]$ (green) and $\bsy{\beta}^{[0]}_5\!=\![5,1]$ (magenta). It can be appreciated that orbits starting at on one side of the line ${\cal S}_b$, converge to it on the same side, without crossing the semi-planes. Also, Lemma \ref{lem:singletransition} and Corollary \ref{recta-S} are illustrated: The blue orbit initialized at  $\bsy{\beta}^{[0]}_2\!=\!\bsy{\beta}^{[0]}_1\!+\!\bsy{\lambda}$, i.e. shifted from the red one (in the straight line direction), converges with the same shift, to $\bsy{\beta^*_2}\!=\!\bsy{\beta^*_1}\!+\!\bsy{\lambda}$. In addition, the convergence of the algorithm is really fast, managing to converge in at most 5 iterations.

\begin{figure}[ht]
\centering\includegraphics[scale=0.4]{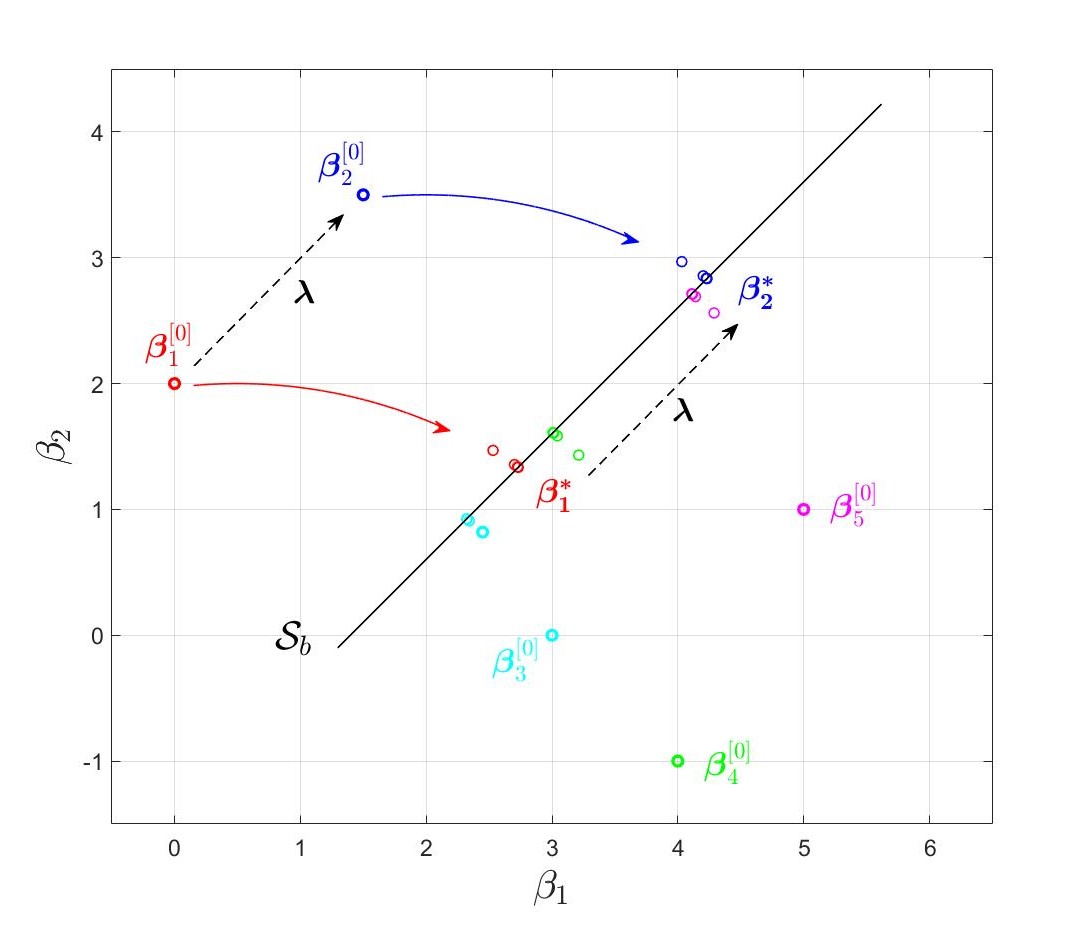}
  \caption{Example with $K\!=\!2$. Orbits of five different i.c.: $\bsy{\beta}^{[0]}_1\!=\![0,2]$ (red), $\bsy{\beta}^{[0]}_2\!=\![1.5,3.5]$ (blue), $\bsy{\beta}^{[0]}_3\!=\![3,0]$ (cyan), $\bsy{\beta}^{[0]}_4\!=\![4,-1]$ (green) and $\bsy{\beta}^{[0]}_5\!=\![5,1]$ (magenta). Only five points of each orbit were plotted due to rapid convergence. Also the line of fixed points, ${\cal S}_b$, with $b\!=\!-1.39726$. The shift vector is $\bsy{\lambda}\!=\![1.5,1.5]$}
  \label{fig:ej-1-2D}
\end{figure}

\subsection{Application to Language models with K=3}

For the second example, we repeat the same procedure for a classification task with $K\!=\!3$. We choose the MNLI dataset \cite{mnli}, which contains examples of sentence pairs with textual entailment annotations. Specifically, given a premise sentence and a hypothesis sentence, the task is to predict whether the premise entails the hypothesis (entailment), contradicts the hypothesis (contradiction), or neither (neutral). The prompt template used for this task was the following: \textit{``Premise: }\texttt{[premise]}\textit{. Hypothesis: }\texttt{[hypothesis]}\textit{. Entailment, neutral or contradiction? Answer:''} and the model used to obtain the probabilities was the same as the first experiment.

In this case, the parameter values are $N_1\!=\!1407$, $N_2\!=\!1298$ and $N_3\!=\!1295$. Fig. \ref{fig:ej-3-3D} shows the line of fixed points, ${\cal S}(\bsy{\beta^*})$, with direction vector $[1,1,1]$; the orbits of five different i.c.: $\bsy{\beta}^{[0]}\!=\![1,-1,1]$ (red), $\bsy{\beta}^{[0]}\!=\![2,-2,2]$ (blue), $\bsy{\beta}^{[0]}\!=\![2.5,-2,1]$ (magenta), $\bsy{\beta}^{[0]}\!=\![2.5,-0.5,1]$ (green) and $\bsy{\beta}^{[0]}\!=\![1,-1.5,1.5]$ (cyan).

\begin{figure}[ht]
\centering\includegraphics[scale=0.4]{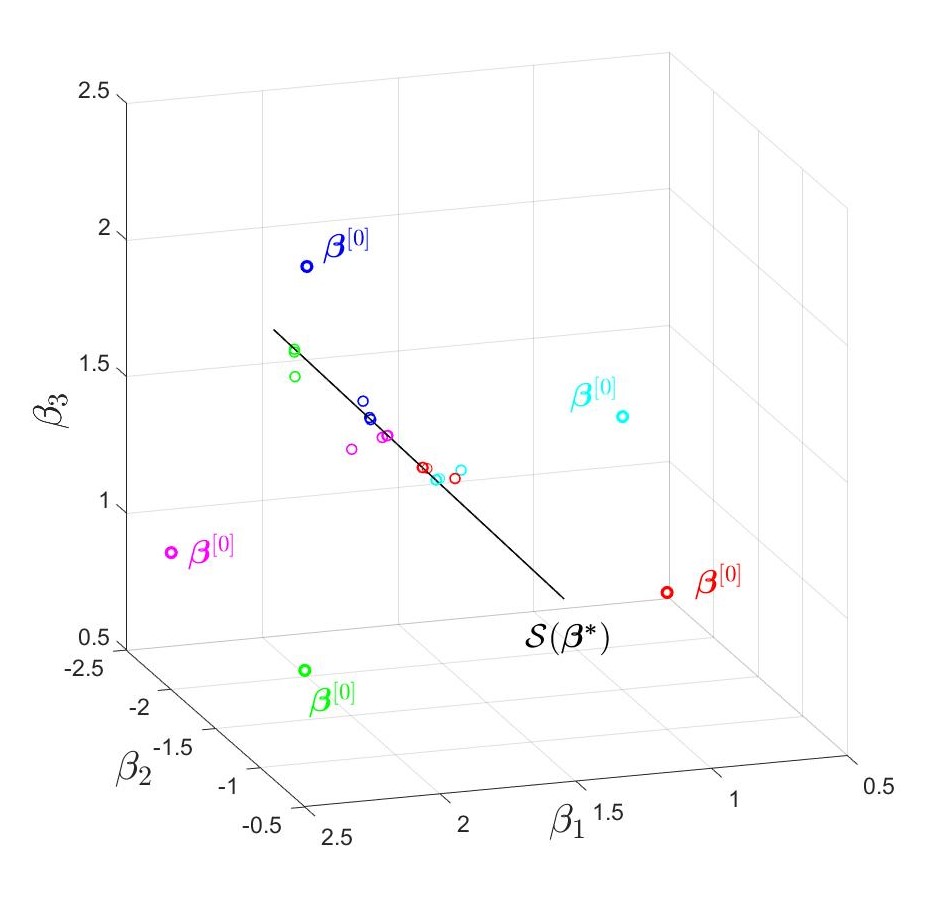}
  \caption{Example with $K\!=\!3$. Orbits of five different i.c.: $\bsy{\beta}^{[0]}\!=\![1,-1,1]$ (red), $\bsy{\beta}^{[0]}\!=\![2,-2,2]$ (blue), $\bsy{\beta}^{[0]}\!=\![2.5,-2,1]$ (magenta), $\bsy{\beta}^{[0]}\!=\![2.5,-0.5,1]$ (green) and $\bsy{\beta}^{[0]}\!=\![1,-1.5,1.5]$ (cyan). Five points of each orbit are plotted. Also the line of fixed points, ${\cal S}(\bsy{\beta^*})$}
  \label{fig:ej-3-3D}
\end{figure}

As we mention in Conj. \ref{conj:convergencia} and \ref{conj:unicidad}, all i.c. converges to the unique straight line of unitary slope. This example would seem to indicate that the characteristics shown for $K\!=\!2$ can be extended to a greater number of classes. In addition, the convergence of the algorithm is really fast, managing to converge in at most 5 iterations.

\subsection{Application to Image Classification}\label{sec:image-classification}

The final example shows that the features shown in the previous examples occur in other completely different applications. In this case, a recognizer of images of dogs and cats was trained on the Asirra dataset \cite{dogs-vs-cats} obtained from the Dogs-vs-Cats Kaggle competition website\footnote{\texttt{https://www.kaggle.com/c/dogs-vs-cats}}. Here, the pretrained ResNet-18 \cite{resnet} model was fine-tuned on 80\% of the training data and validated on the rest. A final accuracy of 0.98 on the validation split was obtained after one epoch using Adam optimization with defaults parameters, a learning rate of 1e-4 and batch size of 64.
\begin{figure}[ht]
\centering\includegraphics[scale=0.4]{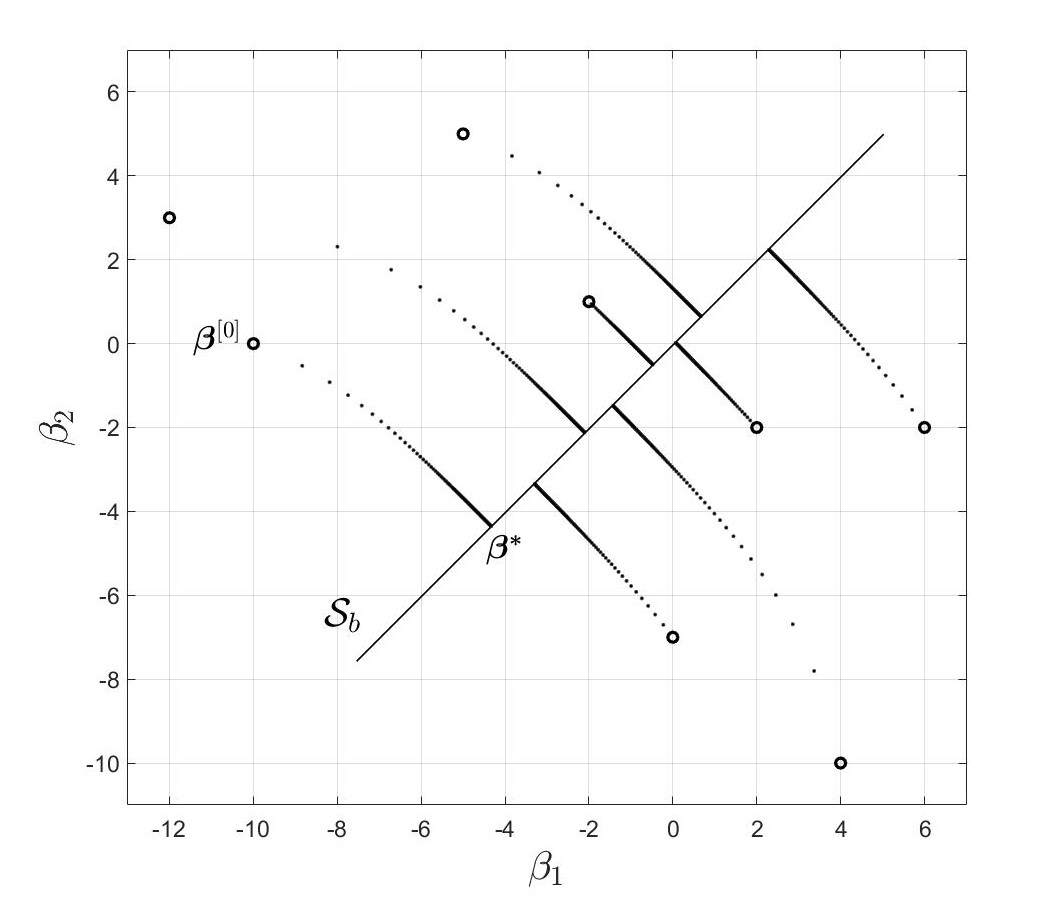}
  \caption{Example with $K\!=\!2$ in an image classification task. Orbits of eight different i.c. are plotted. At least $200$ points of each one were needed to reach the corresponding limit point $\bsy{\beta^*}$}
  \label{fig:ej-dogvscat}
\end{figure}

In this case, the parameter values are $N_1\!=\!10053$, $N_2\!=\!9947$, and $b\!=\!-0.03465$ is obtained. Fig. \ref{fig:ej-dogvscat} shows the line of fixed points and the orbits of different i.c. The same conclusions as in the previous cases can be seen in this example with the difference that up to $200$ steps were needed to achieve convergence. 


\section{Conclusions}\label{sec:conclutions}

In this work we proved several convergences properties of SUCPA, a calibration algorithm described by a non-hyperbolic {map with a non-bounded set of non-isolated fixed points}. Among them it is worth highlighting, on the one hand, the main results: Corollary \ref{recta-S} and Lemma \ref{transition-matrix}. The first one {establishes} that if there exists a fixed point, {then, all points in the straight line of direction $[1,\hdots,1]$ and passing through it must be also fixed points of the map (an unbounded set of non-isolated fixed points)}; while the second one shows that the Jacobian matrix of the system is a regular transition probability matrix (for all $\bsy{\beta}\!\in\!\mathbb{R}^K$) which makes it a non-hyperbolic problem. 

On the other hand, stronger conclusions were proved for the two--classes case {by means of Theo. \ref{main-result-1} which proves the existence and uniqueness of the fixed points straight line}, and Theo. \ref{main-result-2} which shows that every i.c. converges. These results seem to be valid for the general case, but it remains as future work to prove it formally. Additionally, {interesting real-world} application examples were presented in language models and image classification {for which} all the mentioned {results} are experimentally corroborated.

Another possible study to continue this work is the sensitivity of the priors. SUCPA algorithm greatly degrades its performance when the proportion of labels is poorly estimated. Perhaps this phenomenon can be explained from the point of view of convergence, studying how the fixed points change according to variations in the $N_k$ values. This sensitivity analysis is under study and will be reported elsewhere.

\bigskip
\section{Acknowledgments}
This work was supported by the University of Buenos Aires (grant UBACYT 20020220200045BA, 20020220400162BA and 20020190100032BA), CONICET (grant PIP 11220200101826CO) and the ANPCyT (grant PICT 2020-01336).


\bigskip\bigskip\bigskip\bigskip

\bibliographystyle{mystyle}
\bibliography{bibfile}

\begin{thebibliography}{10}

\bibitem{AAH}
Akimoto Y., Auger A., Hansen N.
\newblock An {ODE} method to prove the geometric convergence of adaptive stochastic algorithms.
\newblock {\em Stochastic Processes and their Applications}, 145:269--307, 2022.

\bibitem{Bauerschmidt}
Bauerschmidt R., Brydges D.~C., Slade G.
\newblock Structural stability of a dynamical system near a non-hyperbolic fixed point.
\newblock {\em Ann. Henri Poincaré}, 16:1033–1065, 2015.

\bibitem{Brmmer2010MeasuringRA}
Br{\"u}mmer N.
\newblock {\em Measuring, refining and calibrating speaker and language information extracted from speech}.
\newblock PhD thesis, Stellenbosch University, 2010.

\bibitem{Carr}
Carr J.
\newblock {\em Applications of Centre Manifold Theory}.
\newblock Springer-Verlag, New York, Heidelberg, Berlin, 1982.

\bibitem{Dannan}
Dannan F.~M., Elaydi S.~N., Ponomarenko V.
\newblock Stability of hyperbolic and nonhyperbolic fixed points of one-dimensional maps.
\newblock {\em Journal of Difference Equations and Applications}, 9(5):449--457, 2003.

\bibitem{Devaney}
Devaney R.
\newblock {\em An Introduction to Chaotic Dynamical Systems}.
\newblock Addison-Wesley, 1987.

\bibitem{duda}
Duda R., Hart P., Stork D.
\newblock {\em Pattern Classification}.
\newblock John Wiley, 2 edition, 2001.

\bibitem{Elaydi1}
Elaydi S.
\newblock {\em An Introduction to Difference Equations}.
\newblock Springer, third edition, 2005.

\bibitem{Elaydi2}
Elaydi S.
\newblock {\em Discrete Chaos with Applications in Sciences and Engineering}.
\newblock Taylor \& Francis Group, LLC, second edition, 2007.

\bibitem{Elsayed}
Elsayed E.~M.
\newblock Dynamics and behavior of a higher order rational difference equation.
\newblock {\em The Journal of Nonlinear Science and Applications}, 9(4):1463–1474, 2016.

\bibitem{dogs-vs-cats}
Elson J., Douceur J., Howell J., Saul J.
\newblock Asirra: A captcha that exploits interest-aligned manual image categorization.
\newblock In {\em Proceedings of the 14th ACM Conference on Computer and Communications Security}, CCS '07, page 366–374, New York, NY, USA, 2007. Association for Computing Machinery.

\bibitem{sucpa_arxiv}
Estienne L., Ferrer L., Vera M., Piantanida P.
\newblock Unsupervised calibration through prior adaptation for text classification using large language models.
\newblock {\em ArXiv}, abs/2307.06713, 2023.

\bibitem{survey_calibration}
Filho T., Song H., Perello-Nieto M., Santos-Rodriguez R., Kull M., Peter F.
\newblock Classifier calibration: a survey on how to assess and improve predicted class probabilities.
\newblock {\em Machine Learning}, 112(9):3211--3260, 2023.

\bibitem{godau}
Godau P., Kalinowski P., Christodoulou E., Reinke A., Tizabi M., Ferrer L., J{\"a}ger P., Maier-Hein L.
\newblock Deployment of image analysis algorithms under prevalence shifts.
\newblock In {\em Proc. of MICCAI}, 2023.

\bibitem{guo17}
Guo C., Pleiss G., Sun Y., Weinberger K.
\newblock On calibration of modern neural networks.
\newblock In {\em Proceedings of the International Conference on Machine Learning {ICML}}, Sydney, 2017.

\bibitem{resnet}
He~K., Zhang X., Ren S., Sun J.
\newblock Deep residual learning for image recognition.
\newblock {\em CoRR}, abs/1512.03385, 2015.

\bibitem{Ibrahim}
Ibrahim T.~F., Khan A.~Q., Oğul B., Şimşek D.
\newblock Closed-form solution of a rational difference equation.
\newblock {\em Advances in Stability and Control of Dynamical Systems}, 2021(ID 3168671):12 pages, 2021.

\bibitem{Jamieson}
Jamieson W., Merino O.
\newblock Local dynamics of planar maps with a non-isolated fixed point exhibiting 1-1 resonance.
\newblock {\em Advances in Difference Equations, vol.142}, 2018.

\bibitem{Kapcak}
Kapçak S.
\newblock {\em A Note on Non-hyperbolic Fixed Points of One-Dimensional Maps}, pages 257--267.
\newblock Springer, 2021.

\bibitem{Khan}
Khan A.~Q., Ibrahim T.~F.
\newblock Stability and bifurcations analysis of discrete dynamical systems.
\newblock {\em Discrete Dynamics in Nature and Society}, 2019(ID 8474706), 2019.

\bibitem{Kuznetsov}
Kuznetsov Y.~A.
\newblock {\em Elements of Applied Bifurcation Theory}.
\newblock Springer, New York, 2004.

\bibitem{LiCh}
Li~S., Cheah C.
\newblock Transfer learning algorithm for image classification task and its convergence analysis.
\newblock In {\em IECON 2023- 49th Annual Conference of the IEEE Industrial Electronics Society}, 2023.

\bibitem{Liu-Chen}
Liu Z., Chen G.
\newblock On area-preserving non-hyperbolic chaotic maps: A case study.
\newblock {\em Chaos, Solitons \& Fractals}, 16(5):811--818, 2003.

\bibitem{meyer00}
Meyer C.
\newblock {\em Matrix Analysis and Applied Linear Algebra}.
\newblock SIAM, 2000.

\bibitem{Psarros}
Psarros N., Papaschinopoulos G., Schinas C.~J.
\newblock Study of the stability of a 3x3 system of difference equations using centre manifold theory.
\newblock {\em Applied Mathematics Letters}, 64:185--192, 2017.

\bibitem{gpt2}
Radford A., Wu~J., Child R., Luan D., Amodei D., Sutskever I.
\newblock Language models are unsupervised multitask learners, 2019.

\bibitem{Schaumann}
Schaumann F.
\newblock {\em Investigation of nonhyperbolic dynamical systems}.
\newblock PhD thesis, Max-Planck-Institut für Meteorologie, 2019.

\bibitem{sst2}
Socher R., Perelygin A., Wu~J., Chuang J., Manning C.~D., Ng~A., Potts C.
\newblock Recursive deep models for semantic compositionality over a sentiment treebank.
\newblock In {\em Proceedings of the 2013 Conference on Empirical Methods in Natural Language Processing}, pages 1631--1642, Seattle, Washington, USA, October 2013. Association for Computational Linguistics.

\bibitem{Stevic}
Stević S., Diblík J., Iričanin B., Šmarda Z.
\newblock On a solvable system of rational difference equations.
\newblock {\em Journal of Difference Equations and Applications}, 20(5-6):811--825, 2014.

\bibitem{Tar}
Tarlowski D.
\newblock On asymptotic convergence rate of evolutionary algorithms.
\newblock Technical report, https://doi.org/10.36227/techrxiv.20103404.v2, 2022.

\bibitem{Urbanski}
Urbański M., Zdunik A.
\newblock Geometry and ergodic theory of non-hyperbolic exponential maps.
\newblock {\em Transactions of the American Mathematical Society}, 359(8):3973--3997, 2007.

\bibitem{Wiggins}
Wiggins S.
\newblock {\em Introduction to Applied Nonlinear Dynamical Systems and Chaos, 2nd. Ed.}
\newblock Springer-Verlag, 2003.

\bibitem{mnli}
Williams A., Nangia N., Bowman S.
\newblock A broad-coverage challenge corpus for sentence understanding through inference.
\newblock In {\em Proceedings of the 2018 Conference of the North {A}merican Chapter of the Association for Computational Linguistics: Human Language Technologies, Volume 1 (Long Papers)}, pages 1112--1122, New Orleans, Louisiana, June 2018. Association for Computational Linguistics.

\bibitem{Zolfaghari}
Zolfaghari-Nejad M., Charmi M., Hassanpoor H.
\newblock A new chaotic system with only nonhyperbolic equilibrium points: Dynamics and its engineering application.
\newblock {\em Complexity}, 2022:16 pages, 2022.

\end{thebibliography}

\bigskip
\begin{figure}[ht!]
\noindent\rule{\textwidth}{0.5pt}\\[-.8\baselineskip]\rule{\textwidth}{0.5pt}\\

\begin{minipage}[b]{0.5\linewidth}
\vspace{0.1cm}
\qquad\qquad\qquad\epsfig{file=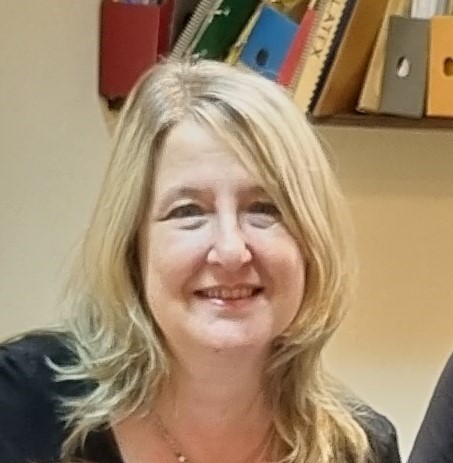,height=1.71in,width=1.6in} 
\end{minipage}
\begin{minipage}[b]{7cm}
\bigskip \bigskip\footnotesize{\textbf{Roberta Hansen}
       - graduated from Faculty of Exact and Natural Sciences of the University of Buenos Aires. She received her Ph.D. degree in Mathematics from the same university in 2009. At the present she is Associate Professor at the Department of Mathematics of the Faculty of Engineering, UBA. and a member of its Nonlinear Dynamics Group. Dr. Hansen's research interests also include fractal geometry, multifractality and complexity in time series.}\end{minipage}\\
\end{figure}
\bigskip
\begin{figure}[ht!]
\begin{minipage}[b]{0.5\linewidth}
\vspace{0.1cm}
\qquad\qquad\qquad\epsfig{file=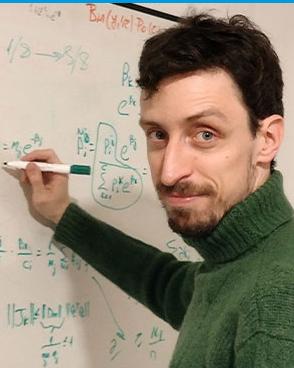,height=1.71in,width=1.6in}
\end{minipage}
\begin{minipage}[b]{7cm}
\bigskip \bigskip\footnotesize{\textbf{Matias Vera}
        - graduated from Engineering School of the University of Buenos Aires in 2014. He received his Ph.D. degree on Engineering (summa cum laude) from the University of Buenos Aires in 2020. Presently, he is an Associate Professor at the University of Buenos Aires and member of the National Scientific and Technical Research Council in Argentina. Dr. Vera's research interests include machine learning, information theory and statistics.}\end{minipage}\\

\end{figure}
\bigskip
\begin{figure}[ht!]
\begin{minipage}[b]{0.5\linewidth}
\vspace{0.1cm}
\qquad\qquad\qquad\epsfig{file=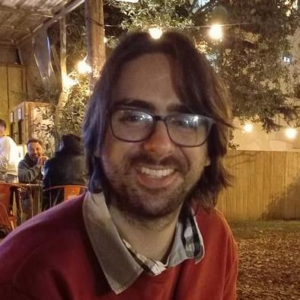,height=1.71in,width=1.6in} 
\end{minipage}
\begin{minipage}[b]{7cm}
\bigskip \bigskip\footnotesize{\textbf{Lautaro Estienne}
        - graduated from Engineering School of the University of Buenos Aires in 2022. He is currently a Ph.D. student at the Engineering School of the University of Buenos Aires and works in Probabilistic Machine Learning, Natural Language and Speech Processing and Calibration.}\end{minipage}\\
\end{figure}
\bigskip
\begin{figure}[ht!]

\begin{minipage}[b]{0.5\linewidth}
\vspace{0.1cm}
\qquad\qquad\qquad\epsfig{file=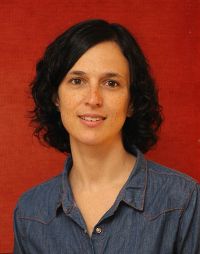,width=1.6in} 
\end{minipage}
\begin{minipage}[b]{7cm}
\bigskip \bigskip\footnotesize{\textbf{Luciana Ferrer}
       - graduated from Engineering School of the University of Buenos Aires, Argentina, in 2001. She received her Ph.D. degree in Electronic Engineering from Stanford University, USA, in 2009. She is currently a researcher at the Computer Science Institute, affiliated to the University of Buenos Aires (UBA) and the National Scientific and Technical Research Council (CONICET), Argentina. Her primary research area is machine learning applied to speech, audio, and language processing tasks, with special focus on the problem of score calibration and robustness to domain mismatch. }\end{minipage}\\
\end{figure}
\bigskip
\begin{figure}[ht!]

\begin{minipage}[b]{0.5\linewidth}
\vspace{0.1cm}
\qquad\qquad\qquad\epsfig{file=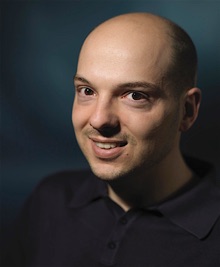,height=1.71in,width=1.6in}
\end{minipage}
\begin{minipage}[b]{7cm}
\bigskip \bigskip\footnotesize{\textbf{Pablo Piantanida}
received the  B.Sc.  degree in electrical engineering and the M.Sc. degree from the University of Buenos Aires, Argentina,   in   2003,   and the   Ph.D.   degree   from Université Paris-Sud, Orsay, France, in 2007.  He is currently director of the International Laboratory and Learning Systems (ILLS), professor at CentraleSupélec (Université Paris-Saclay) with CNRS and Associate Academic Member Mila - Quebec AI Institute (Mila). From 2019 to  2021, he was also an associate member of Comète – Inria research team (Lix - Ecole Polytechnique). In 2018 and 2019, he was vising professor at the Université de Montréal and Laboratoire de Mathématiques d'Orsay (LMO). His research interests include information theory, machine learning, security of learning systems and the secure processing of information and applications to computer vision, health, natural language processing, among others. He has served as the General Co-Chair for the 2019 IEEE  International  Symposium on  Information  Theory  (ISIT).  He served as an  Associate  Editor for the  IEEE  TRANSACTIONS  ON  INFORMATION FORENSICS AND SECURITY and Editorial Board of Section "Information Theory, Probability and Statistics" for Entropy. He is member of the IEEE Information Theory Society Conference Committee. 
}\end{minipage}\\
\end{figure}
\end{document}